\documentclass[runningheads]{llncs}
\usepackage{graphicx}

\usepackage{multirow}
\usepackage{tikz}
\usepackage{comment}
\usepackage{url}   
\usepackage[pagebackref=true,breaklinks=true,colorlinks,bookmarks=false]{hyperref} 

\usepackage{amsmath}
\usepackage{amssymb}

\usepackage{color}
\usepackage{bbding}

\usepackage{algorithm}
\usepackage{algpseudocode}

\usepackage{float}
\usepackage{booktabs} 
\usepackage{diagbox}

\def\eg{\textit{e.g.}}

\newcommand{\fullmodelname}{preconditioned diffusion sampling}
\newcommand{\shortmodelname}{PDS}

\usepackage{bm}

\begin{document}

\pagestyle{headings}
\mainmatter

\title{Accelerating Score-based Generative Models with Preconditioned Diffusion Sampling}

\titlerunning{Accelerating SGMs with Preconditioned Diffusion Sampling} 

\author{
  Hengyuan Ma\inst{1} \and
  Li Zhang\inst{1}\thanks{Li Zhang (lizhangfd@fudan.edu.cn) is the corresponding author with School of Data Science, Fudan University.
  H. Ma and J. Feng are with Institute of Science and Technology for Brain-inspired Intelligence, Fudan University.
  X. Zhu is with Surrey Institute for People-Centred Artificial Intelligence, CVSSP, University of Surrey.} \and
  Xiatian Zhu\inst{2}\and
  Jianfeng Feng\inst{1} 
  \vspace{-0.5em} 
}

\authorrunning{H. Ma, L. Zhang, et al.} 

\institute{ 
Fudan University
\and
University of Surrey
\\
\vspace{0.5em} 
\url{https://github.com/fudan-zvg/PDS}
}



\maketitle

\begin{abstract}
Score-based generative models (SGMs) have recently emerged as a promising class of generative models.
However, a fundamental limitation is that their inference is very slow
due to a need for many (\eg, $2000$) iterations of sequential computations.
An intuitive acceleration method is to reduce the sampling iterations which however causes severe performance degradation.
We investigate this problem by
viewing the diffusion sampling process
as a Metropolis adjusted Langevin algorithm, which helps reveal the underlying cause to be ill-conditioned curvature.
Under this insight, we propose a model-agnostic {\bf\em preconditioned diffusion sampling} (PDS) method that leverages matrix preconditioning to alleviate the aforementioned problem.
Crucially, PDS is proven theoretically to converge to the original target distribution of a SGM, no need for retraining.
Extensive experiments on three image datasets with a variety of resolutions and diversity validate that PDS consistently accelerates off-the-shelf SGMs whilst maintaining the synthesis quality. In particular, PDS can accelerate by up to $29\times$ on more challenging high resolution (1024$\times$1024) image generation.

\end{abstract}

\begin{figure}[ht]
    \centering
    \begin{minipage}{1\linewidth}
        \textcolor{white}{----------}$T = 2000$\textcolor{white}{----------}$T = 200$\textcolor{white}{----------}$T = 133$\textcolor{white}{----------}$T = 100$\textcolor{white}{----------}$T = 66$
    \end{minipage}
    \rotatebox{90}{\textcolor{white}{------} \quad\quad \quad\quad  Ours \quad\quad\quad\quad\quad\quad\quad\quad\quad Baseline~\cite{song2020score}} 
    \includegraphics[scale=0.15]{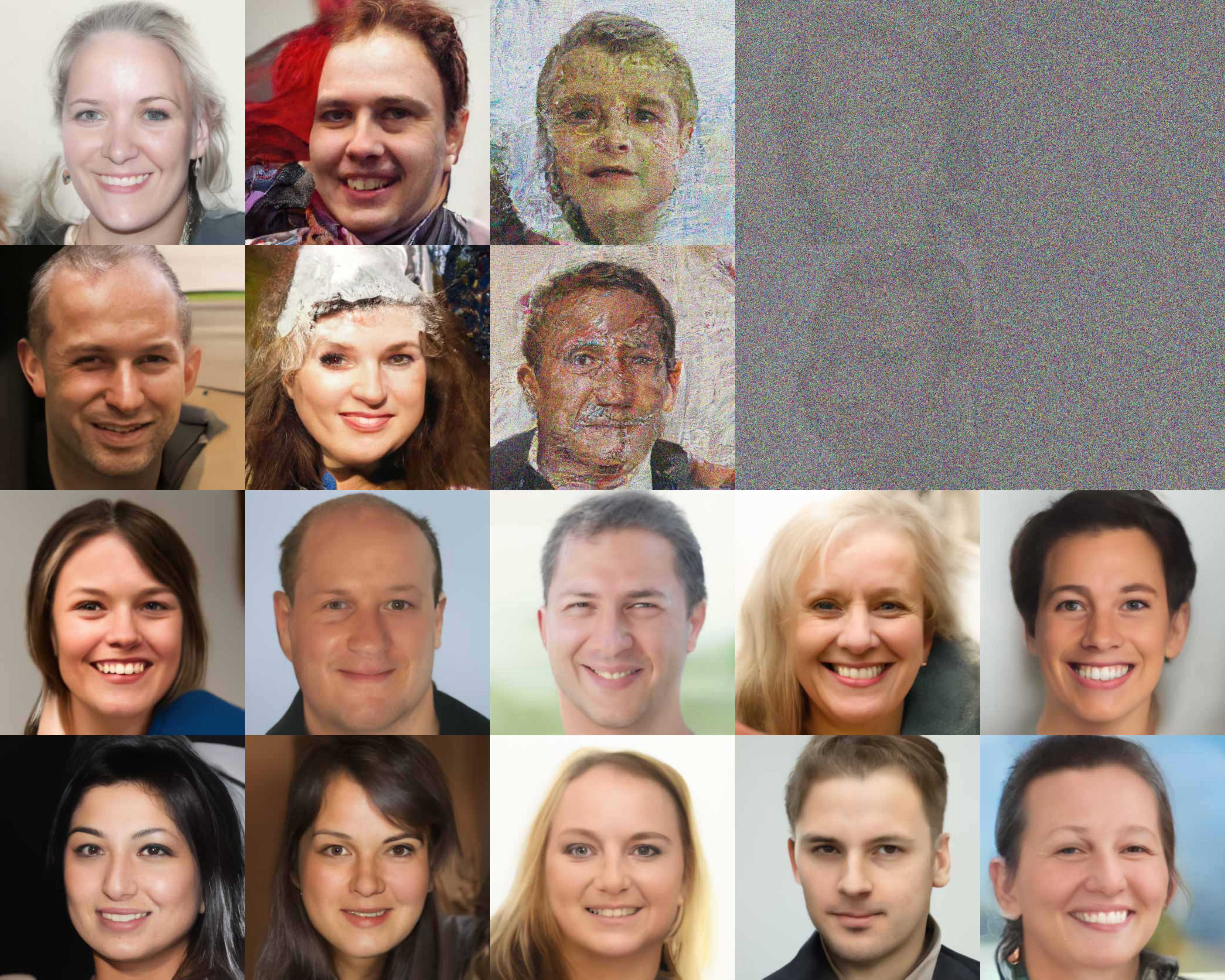}
    \caption{
        Facial images at a resolution of $1024\times 1024$ generated by NCSN++~\cite{song2020score}
        under a variety of sampling iterations
        (top) without and (bottom) with our PDS. 
        It is evident that NCSN++ decades quickly with increasingly reduced sampling iterations,
        which can be well solved with PDS.
        In terms of running speed for generating a batch of 8 images, PDS reduces the time cost from 2030 seconds (the sampling iterations $T=2000$) 
    		to $71$ seconds ($T=66$) on one NVIDIA RTX 3090 GPU,
    		which delivers $29\times$ acceleration.
    		Dataset: FFHQ \cite{karras2019style}.
    		More samples in \ref{sec:example_sm}.
    		}
        \label{fig:ffhq}
\end{figure}

\section{Introduction}
As an alternative framework to generative adversarial networks (GANs)~\cite{goodfellow2014generative}, recent score-based generative models (SGMs)~\cite{song2019generative,DBLP:conf/nips/0011E20,song2020score,song2021maximum} have demonstrated excellent abilities in data synthesis (especially in high resolution images) with easier optimization \cite{song2019generative}, richer diversity \cite{xiao2021tackling}, and more solid theoretic foundation \cite{de2021diffusion}. 
Starting from a sample initialized with a Gaussian distribution, a SGM produces a target sample by simulating a diffusion process, typically a Langevin dynamics.
Compared to the state-of-the-art GANs \cite{DBLP:conf/iclr/BrockDS19,karras2019style,DBLP:conf/iclr/KarrasALL18}, a significant drawback with existing SGMs is {\em drastically slower generation} due to the need of taking many iterations for a sequential diffusion process \cite{song2020score,nichol2021glide,xiao2021tackling}. 
Formally, the discrete Langevin dynamic for sampling is typically formulated as
\begin{align}\label{eq: Langevin_discrete}
    \mathbf{x}_{t} = \mathbf{x}_{t-1} +  \frac{\epsilon_t^2}{2}\bigtriangledown_{\mathbf{x}}\log p^{\ast}(\mathbf{x}_{t-1})  + \epsilon_t \mathbf{z}_t, 1\leq t \leq T
\end{align}
where $\epsilon_t$ is the step size (a positive real scalar), $\mathbf{z}_t$ is an independent standard Gaussian noise, and $T$ is the iteration number. Starting from a standard Gaussian sample $\mathbf{x}_0$, with a total of $T$ steps this sequential sampling process gradually transforms $\mathbf{x}_0$ to the sample $\mathbf{x}_T$ that obeys the target distribution $p^{\ast}$. 
Often, $T$ is at the scale of 1000s, and the entire sampling process is lengthy.

For accelerating the sampling process, a straightforward method is to reduce $T$ by a factor and proportionally expand $\epsilon_t$ simultaneously, so that the number of calculating the gradient $\bigtriangledown_{\mathbf{x}}\log p^{\ast}(\mathbf{x})$, which consumes the major time,
decreases whilst keeping the total update magnitude. 
However, this often makes pretrained 
SGMs fail in image synthesis.
In general, we observe two types of failure:
insufficient detailed structures (left of Fig.~\ref{fig: mnist} and Fig.~\ref{fig:church_ncsnv2}),
and dazzling with heavy noises (left of Fig.~\ref{fig:ffhq} and Fig.~\ref{fig: lsun_ncsnpp}).
Conceptually, the sampling process as defined in Eq.~\eqref{eq: Langevin_discrete} can be considered as a special case of Metropolis adjusted Langevin algorithm (MALA) at the Metropolis-Hastings rejection probability of 100\% \cite{roberts2002langevin,welling2011bayesian,girolami2011riemann}.
When the coordinates of a target sample (\eg,
the pixel locations of a natural image) are strongly correlated,
the isotropic Gaussian noises $\{\mathbf{z}_t\}$ would become {\em inefficient} for the variables $\mathbf{x}$, caused by the {\em ill-conditioned curvature} of the sampling process \cite{girolami2011riemann}. 

\begin{figure}[h]
    \centering
    \includegraphics[scale=1.3]{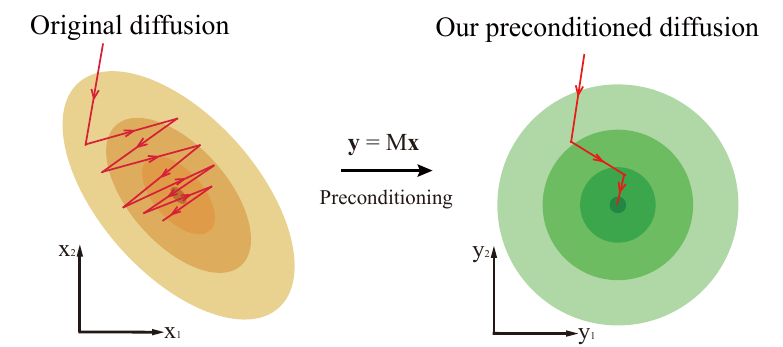}
    \caption{Illustration of the preconditioning method for accelerating sampling process.}
    \label{fig:demo}
\end{figure}

In light of this insight as above, we propose 
an efficient, model-agnostic 
{\bf\em \fullmodelname}~(\shortmodelname) method
for accelerating existing pretrained SGMs without the need for model retraining.
The key idea is that mathematically {\em matrix preconditioning} is effective in substituting a transformation variable in a way that the rates of curvature become more similar along all the directions~\cite{roberts2002langevin,li2016preconditioned}, hence solving the ill-conditioned curvature problem.
Formally, we enrich the above Langevin dynamics (Eq. \eqref{eq: Langevin_discrete}) by imposing a preconditioning operation into the diffusion process as
\begin{align}\label{eq: Langevin_discrete_M}
    \mathbf{x}_{t} = \mathbf{x}_{t-1} +  \frac{\epsilon_t^2}{2}MM^{\mathsf{T}}\bigtriangledown_{\mathbf{x}}\log p^{\ast}(\mathbf{x}_{t-1})  + \epsilon_t M\mathbf{z}_t,
\end{align}
where $M$ is the newly introduced preconditioning matrix designed particularly for regulating the behavior of accelerated diffusion processes.
Concretely, this proposed reformulation equips the diffusion process with 
a novel ability to {\em enhance} or {\em restrain} the generation of detailed structures via controlling the different frequency components\footnote{
More theoretical explanation on why {\em directly} regulating the frequency domain of a diffusion process is possible is provided in \ref{sec:freq_sm} .} of the noises~\cite{BOVIK200997}.
This can be realized in the single formulation (Eq. \eqref{eq: Langevin_discrete_M}) with the $M$ matrix designed flexibly to tackle both failure cases. Crucially, according to the theorems with Fokker-Planck equation~\cite{gardiner1985handbook} our PDS can preserve the original SGM's target distribution.
Further, any structured priors available with a target distribution can be also accommodated, \eg,
the average spatial structures of human faces.
The computational cost of calculating $M$ is marginal when using Fast Fourier Transform (FFT)~\cite{brigham1988fast}.
In this work, we make the following {\bf contributions}:
{\bf (1)} 
We investigate the low inference efficiency problem of off-the-shelf SGMs for high-resolution image synthesis, which is critical yet under-studied in the literature.
{\bf (2)} For sampling acceleration,
we introduce a novel \fullmodelname
~(\shortmodelname) process.
\shortmodelname~reformulates the existing diffusion process
with a preconditioning operation additionally imposed for adaptively regulating the frequency components' amount in the noises,
whilst keeping the original target distributions in convergence.
{\bf (3)}
With \shortmodelname, a variety of pretrained SGMs can be accelerated significantly for image synthesis of various spatial resolutions, without model retraining.
In particular, \shortmodelname~delivers $29\times$ reduction in wall-clock time for high-resolution image synthesis.

\section{Related work}
Sohl-Dickstein et al. \cite{sohl2015deep} first proposed to destroy the data distribution through a diffusion process slowly
and learned the backward process to recover the data, inspired by non-equilibrium statistical physics. 
Later on, Song and Ermon \cite{song2019generative} further explored SGMs by introducing the noise conditional score network (NCSN). 
Song and Ermon \cite{DBLP:conf/nips/0011E20} proposed NCSNv2 that scaled
NCSN for higher resolution image generation (\eg, $256\times 256$)
by scaling noises and improving stability with moving average.
Song et al. \cite{song2020score} summarized all the previous SGMs into a unified framework based on the stochastic differential equation (SDE) and proposed the NCSN++ model to generate high-resolution images via numerical SDE solvers for the first time. Bortoli et al. \cite{de2021diffusion} provided the first quantitative convergence results for SGMs. Vahdat et al. \cite{vahdat2021score} developed Latent Score-based Generative Model (LSGM) that trains SGMs in a latent space with the variational autoencoder framework. 
Another class of relevant generative models, mainly trained by reducing an evidence lower bound (ELBO) called denoising diffusion probabilistic models (DDPMs)~\cite{DBLP:conf/nips/HoJA20,DBLP:conf/icml/NicholD21,song2020denoising,dhariwal2021diffusion,ho2021cascaded,nichol2021glide,bao2022analytic}, also demonstrate excellent performance on image synthesis.
Commonly, all of the above works use isotropic Gaussian distributions for the diffusion sampling.

Recently there are some works proposed on accelerating SGMs.
Dockhorn et al. \cite{dockhorn2021score} improved the SGMs with Hamiltonian Monte Carlo methods~\cite{neal2011mcmc} and proposed critically-damped Langevin diffusion (CLD) based SGMs that achieves superior performance.
Jolicoeur-Martineau et al. \cite{jolicoeur2021gotta} utilized a numerical SDE solver with adaptive step sizes to accelerate SGMs.
However, these methods are limited in the following aspects:
{\bf (1)}
They tend to involve much extra computation. For example, CLD based SGMs expand the dimension of data by $2$ times for learning the velocity of the diffusion. Jolicoeur-Martineau et al. 
\cite{jolicoeur2021gotta} added a high-order numerical solver that increases the number of calling the SGM, resulting in much more time. 
In comparison, with our PDS the only extra calculation relates the preconditioning matrix that can be efficiently implemented by Fast Fourier Transform.
{\bf (2)}
They are restricted to a single specific SGM while 
our PDS is model agnostic.
{\bf (3)}
Unlike this work, none of them has demonstrated a scalability to more challenging high-resolution image generation tasks (\eg, FFHQ facial images).
\section{Preliminary}

\paragraph{\bf Scored-based generative models (SGMs).}
Score matching is developed for non-normalized statistical learning~\cite{hyvarinen2005estimation}. Given i.i.d. samples of an unknown distribution $p^{\ast}$, score matching allows the model to directly approximate the \textit{score function} $\bigtriangledown_{\mathbf{x}}\log p^{\ast}(\mathbf{x})$.
SGMs aim to generate samples from $p^{\ast}$ via score matching by simulating a Langevin dynamics initialized by Gaussian noise
\begin{align}\label{eq: reverse}
        d\mathbf{x} = \frac{g^2(t)}{2}\bigtriangledown_{\mathbf{x}}\log p^{\ast}(\mathbf{x})dt + g(t)d\mathbf{w},
\end{align}
where $g: \mathbb{R}^{+}\rightarrow\mathbb{R}^{+}$ controls the step size and $d\mathbf{w}$ represents a Wiener process.
With this process, we transform a sample drawn from an
initial Gaussian distribution
to approach the desired distribution $p^{\ast}$.
A classical SGM, noise conditional score network (NCSN) \cite{song2019generative}, is trained by learning how to reverse a process of gradually corrupting the samples from $p^{\ast}$, and aims to match the score function. After training, NCSN starts from a Gaussian distribution and travels to the target distribution $p^{\ast}$ by simulating an annealed Langevin dynamics.

\paragraph{\bf Recent improvements.}
Song and Ermon \cite{DBLP:conf/nips/0011E20} presented NCSNv2 that improves the original NCSN by designing better noise scales, iteration number, and step size. This new variant is also more stable by using the moving average technique.
Song et al. \cite{song2020score} further proposed NCSN++ that utilizes an existing numerical solver of stochastic differential equations to enhance both the speed of convergence and the stability of the sampling method. Importantly, NCSN++ can synthesize high-resolution images at high quality.

\paragraph{\bf Limitation analysis.}
Although SGMs have been able to generate images comparable to GANs~\cite{goodfellow2014generative}, they are much slower due to the sequential computation during the sampling phase. For example, to produce $8$ facial images at $1024\times 1024$ resolution, a SGM spends more than 30 mins.
To maximize the potential of SGMs,
it is critical to solve this slow inference bottleneck.
\section{Method}
We aim to solve the slow inference problem with SGMs.
For easier understanding, let us start from the most classical Langevin dynamics.
\subsection{Steady-state distribution analysis}\label{sec: steady}
Consider the classical Langevin dynamics
\begin{align}\label{eq: Langevin}
    d\mathbf{x} = \frac{\epsilon^2}{2}\bigtriangledown_{\mathbf{x}}\log p^{\ast}(\mathbf{x})dt + \epsilon d\mathbf{w},
\end{align}
where $p^{\ast}$ is the target distribution, and $\epsilon >0$ is the fixed step size. It is associated with a Fokker-Planck equation
\begin{align}\label{eq: Langevin_fp}
    \frac{\partial p}{\partial t} = - \frac{\epsilon^2}{2}\bigtriangledown_{\mathbf{x}}\cdot (\bigtriangledown_{\mathbf{x}}\log p^{\ast}(\mathbf{x})p)+\frac{\epsilon^2}{2}\Delta_{\mathbf{x}} p,
\end{align}
where $p=p(\mathbf{x},t)$ describes the distribution of $\mathbf{x}$ that evolves over time.
The steady-state solution of Eq.~\eqref{eq: Langevin_fp} corresponds to the probabilistic density function of the steady-state distribution of Eq.~\eqref{eq: Langevin}, i.e., $p^{\ast}$
\begin{align}\label{eq: Langevin_steady}
   \bigtriangledown_{\mathbf{x}}\cdot (\bigtriangledown_{\mathbf{x}}\log p^{\ast}(\mathbf{x})p) = \Delta_{\mathbf{x}} p.
\end{align}

The Fokker-Planck equation tells us how to preserve the steady-state distribution of the original process when we alter Eq.~\eqref{eq: Langevin} for specific motivations. Concretely, we can impose an invertible linear operator $M$ to the noise term $d\mathbf{w}$ and conduct the associated operation on the gradient term so that the steady-state distribution can be preserved. This design is formulated as:
\begin{align}\label{eq: Langevin_M}
    d\mathbf{x} = \frac{\epsilon^2}{2}(MM^{\mathsf{T}}+S)\bigtriangledown_{\mathbf{x}}\log p^{\ast}(\mathbf{x})dt + \epsilon Md\mathbf{w},
\end{align}
where $S$ is a skew-symmetric linear operator. In fact, we have
\begin{theorem}\label{thm: unchange}
 The steady-state distribution of Eq.~\eqref{eq: Langevin} and Eq.~\eqref{eq: Langevin_M} are the same, as long as the linear operator $M$ is invertible and the linear operator $S$ is skew-symmetric.
\end{theorem}
\begin{proof}
The Fokker-Planck equation of Eq.~\eqref{eq: Langevin_M} is
\begin{align}
        &\frac{\partial p}{\partial t} = - \frac{\epsilon^2}{2}\bigtriangledown_{\mathbf{x}}\cdot (MM^{\mathsf{T} }\bigtriangledown_{\mathbf{x}}\log p^{\ast}(\mathbf{x})p)  \nonumber\\ &-\frac{\epsilon^2}{2}\bigtriangledown_{\mathbf{x}}\cdot  (S\bigtriangledown_{\mathbf{x}}\log p^{\ast}(\mathbf{x})p)+\frac{\epsilon^2}{2}\bigtriangledown_{\mathbf{x}}\cdot (MM^{\mathsf{T} }\bigtriangledown_{\mathbf{x}} p).
\end{align}
The corresponding steady-state equation is
\begin{align}
\bigtriangledown_{\mathbf{x}}\cdot (MM^{\mathsf{T} }\bigtriangledown_{\mathbf{x}}\log p^{\ast}(\mathbf{x})p)+\bigtriangledown_{\mathbf{x}}\cdot  (S\bigtriangledown_{\mathbf{x}}\log p^{\ast}(\mathbf{x})p) =\bigtriangledown_{\mathbf{x}}\cdot (MM^{\mathsf{T} }\bigtriangledown_{\mathbf{x}} p).  
\end{align}
Set $p=p^{\ast}$, the above equation becomes
\begin{align}
    \bigtriangledown_{\mathbf{x}}\cdot (MM^{\mathsf{T} }\bigtriangledown_{\mathbf{x}}\log p^{\ast}(\mathbf{x})p^{\ast})+\bigtriangledown_{\mathbf{x}}\cdot  (S\bigtriangledown_{\mathbf{x}}\log p^{\ast}(\mathbf{x})p^{\ast}) =\bigtriangledown_{\mathbf{x}}\cdot (MM^{\mathsf{T} }\bigtriangledown_{\mathbf{x}} p^{\ast}).  
\end{align}
The first term in the L.H.S. equals to the R.H.S., since
\begin{align}
     \bigtriangledown_{\mathbf{x}}\cdot (MM^{\mathsf{T} }\bigtriangledown_{\mathbf{x}}p^{\ast} \frac{1}{p^{\ast}}p^{\ast})=     \bigtriangledown_{\mathbf{x}}\cdot (MM^{\mathsf{T} }\bigtriangledown_{\mathbf{x}}p^{\ast}).
\end{align}
Additionally, the second term in the L.H.S. equals to zero, since
 $S$ is skew-symmetric. Then, the steady-state solution of Eq.~\eqref{eq: Langevin_steady} also satisfies the steady-state equation of Eq.~\eqref{eq: Langevin_M}. As a result, the theorem is proved.
\end{proof}
We can extend the above results to a more general case as follows.
\begin{theorem}\label{thm: unchange2}
 Consider the diffusion process
\begin{align}\label{eq: dd_process}
    d\mathbf{x} = \frac{1}{2}G(t)G(t)^{\mathsf{T}}\bigtriangledown_{\mathbf{x}}\log p^{\ast}(\mathbf{x})dt  + G(t) d\mathbf{w},
\end{align}
where $G:\mathbb{R} \rightarrow \mathbb{R}^{d\times d}$. $M$ is an invertible $d\times d$ matrix and $S$ is a skew-symmetric $d\times d$ matrix. Denote $p^{\ast}$ as the steady-state distribution of Eq.~\eqref{eq: dd_process}, then the process
\begin{align}\label{eq: dd_process2}
                d\mathbf{x} = \frac{1}{2}MM^{\mathsf{T}}G(t)G(t)^{\mathsf{T}}\bigtriangledown_{\mathbf{x}}\log p^{\ast}(\mathbf{x})dt+S\bigtriangledown_{\mathbf{x}}\log p^{\ast}(\mathbf{x})dt  + MG(t) d\mathbf{w},
\end{align}
has the same steady-state distribution as that of Eq.~\eqref{eq: dd_process}.
\end{theorem}
\begin{proof}
The steady-state distribution of Eq.~\eqref{eq: dd_process} satisfies 
\begin{align}\label{eq: dd_process_steady}
    \bigtriangledown_{\mathbf{x}}\cdot (G(t)G(t)^{\mathsf{T}}\bigtriangledown_{\mathbf{x}}\log p^{\ast}(\mathbf{x})p) =\bigtriangledown_{\mathbf{x}}\cdot (G(t)G(t)^{\mathsf{T}}\bigtriangledown_{\mathbf{x}} p).  
\end{align}
and the steady-state distribution of Eq.~\eqref{eq: dd_process2} satisfies
\begin{align}
    &\bigtriangledown_{\mathbf{x}}\cdot (MM^{\mathsf{T }}G(t)G(t)^{\mathsf{T}}\bigtriangledown_{\mathbf{x}}\log p^{\ast}(\mathbf{x})p)+\bigtriangledown_{\mathbf{x}}\cdot  (S\bigtriangledown_{\mathbf{x}}\log p^{\ast}(\mathbf{x})p) \\
    &=\bigtriangledown_{\mathbf{x}}\cdot (MM^{\mathsf{T} }G(t)G(t)^{\mathsf{T}}\bigtriangledown_{\mathbf{x}} p). 
\end{align}
Using the skew-symmetry of $S$, it is easy to find that $p^{\ast}$
satisfies both Eq.~\eqref{eq: dd_process} and Eq.~\eqref{eq: dd_process2}. Therefore, the theorem is proved.
\end{proof}
\begin{remark}
The conditions of this theorem are all satisfied for the diffusion process used in NCSN, NCSNv2, and NCSN++.
\end{remark}

{\color{black}Thm.~\ref{thm: unchange2} motivates us to design a preconditioning matrix as Eq.~\eqref{eq: Langevin_M} while keeping the steady-state distribution simultaneously.}
{\color{black}This is also because, preconditioning has been proved to be able to significantly accelerate the stochastic gradient descent algorithm (SGD) and Metropolis adjusted Langevin algorithms (MALA)~\cite{roberts2002langevin}.
Besides, SGD provides another view for interpreting our method, that is, SGMs sequentially reduce the energy ($-\log p^{\ast}(\mathbf{x})$) of a sample $\mathbf{x}$ via stochastic gradient descent, with the randomness coming from the Gaussian noises added at every single step.}

\subsection{Preconditioned diffusion sampling}
{\color{black} We study how to construct the preconditioning operator using $M$ to accelerate the sampling phase of SGMs, with $S = 0$ for Eq. \eqref{eq: Langevin_M}.}
It is observed that when reducing the iteration number for the sampling process of a SGM and expand the step size proportionally for a consistent accumulative update, the images generated tend to miss necessary detailed structures (see left of Fig.~\ref{fig: mnist} and Fig.~\ref{fig:church_ncsnv2}), or involve high-frequency noises (left of Fig.~\ref{fig:ffhq} and Fig.~\ref{fig: lsun_ncsnpp}).
These failure phenomena motivates us to leverage a preconditioning operator $M$ serving as a filter to regulate the frequency distribution of the samples.
\begin{enumerate}
    \item Given an input vector $\mathbf{x}$, we first use Fast Fourier Transform (FFT)~\cite{brigham1988fast} to map it into the frequency domain $\hat{\mathbf{x}} = F[\mathbf{x}]$. For images, we adopt the 2D FFT that implements 1D FFT column-wise and row-wise successively.
    
    \item Then we adjust the frequency signal using a mask $R$ in the same shape as $\mathbf{x}$: $R\odot\hat{\mathbf{x}}$, where $\odot$ means element-wise multiplication. The elements of $R$ are all positive.
    
    \item Lastly, we map the vector back to the original space by the inverse of Fast Fourier Transform: $F^{-1}[R\odot\hat{\mathbf{x}}]$. 
\end{enumerate}

For specific tasks (\eg, human facial image generation), most samples might share 
a consistent structural characteristics. This prior knowledge however is unavailable
with the noises added to each step in the diffusion process.
To solve this problem, we further propose a space structure filter $A$ for {\bf \em space preconditioning},
constructed by statistical average of random samples.
This can be used to regulate the noise via element-wise multiplication as:
$A\odot[\cdot]$.
Combining the both operations above,  
we define a preconditioning operator $M$ as 
\begin{align} \label{eq:precondition_op}
    M[\cdot] = A\odot F^{-1}[R\odot F[\cdot]].
\end{align}
To guarantee the invertibility of $M$, we set the elements of $R$ strictly positive. 
For the tasks without clear space structure priors, 
we simply do not apply the space preconditioning by setting all the elements of $A$ to $1$. 
We operate $M$ on the noise term $d\mathbf{w}$ and adjust the gradient term to keep the steady-state distribution as shown in Eq.~\eqref{eq: Langevin_M}, utilizing Thm.~\ref{thm: unchange}.

Interestingly, we found that the proposed method above
is likely to even cause further model degradation.
This is because, if we implement a variable transformation as $\mathbf{y} = M^{-1}\mathbf{x}$,
Eq.~\eqref{eq: Langevin_M} can be rewritten as 
\begin{align}
    d\mathbf{y} = \frac{\epsilon^2}{2}\bigtriangledown_{\mathbf{y}}\log p^{\ast}(\mathbf{y})dt + \epsilon d\mathbf{w},
\end{align}
which returns to the same format as the original process. 
The diffusion process is made worse since, $M^{-1}$, the inverse of $M$, could impose the exactly opposite effect of $M$.
To overcome this challenge, we further substitute $M$ with $M^{-1}$ in Eq.~\eqref{eq: Langevin_M} in order to take the positive effect of $M$ as
\begin{align}\label{eq: accelerate2}
    d\mathbf{x} = \frac{\epsilon^2}{2}M^{-1}M^{\mathsf{-T}}\bigtriangledown_{\mathbf{x}}\log p^{\ast}(\mathbf{x})dt + \epsilon M^{-1}d\mathbf{w}.
\end{align}
Since in this case, we can rewrite Eq.~\eqref{eq: accelerate2} in the same format as the original process, after applying the variable transformation $\mathbf{y} = M\mathbf{x}$. 

\noindent{\bf A general formulation.}
For theory completeness, 
we further briefly discuss the possibility to construct preconditioning matrix using the matrix $S$ (Eq. \eqref{eq: Langevin_M}) as an accelerator of the diffusion process.
{\bf \em Note, this is merely a theoretical extension of our main model PDS as formulated above.}

This is motivated by the theories from \cite{ottobre2016markov,rey2015irreversible,lelievre2013optimal} that the term $S\bigtriangledown_{\mathbf{x}}\log p^{\ast}(\mathbf{x})dt$ drives a solenoidal flow that makes the system converge faster to the steady state.
According to~\cite{hwang2005accelerating}, under the regularity conditions, $\left | \mathbf{x}(t)\right |$ usually does not reach the infinity in a finite time, and the convergence of an {\em autonomous} (the right side of the equation does not contain time explicitly) diffusion process 
\begin{align}
    d\mathbf{x} = \frac{\epsilon^2}{2}\bigtriangledown_{\mathbf{x}}\log p^{\ast}(\mathbf{x})dt + \epsilon d\mathbf{w}
\end{align}
can be accelerated by introducing a vector field $C(\mathbf{x})\in \mathbb{R}^d \rightarrow \mathbb{R}^d$
\begin{align}
    d\mathbf{x} = \frac{\epsilon^2}{2}\bigtriangledown_{\mathbf{x}}\log p^{\ast}(\mathbf{x})dt +C(\mathbf{x})dt  + \epsilon d\mathbf{w},
\end{align}
where $C(\mathbf{x})$ should satisfy
\begin{align}
    \bigtriangledown_{\mathbf{x}}\cdot(\frac{C(\mathbf{x})} {p^{\ast}(\mathbf{x})}) = 0.
\end{align}
It is easy to show that $C(\mathbf{x}) =S\bigtriangledown_{\mathbf{x}}\log p^{\ast}(\mathbf{x})$ satisfies the above condition.
However, the diffusion process of existing SGMs is typically {\em not autonomous}, due to the step size $\epsilon$ varies across time designed to guarantee numerical stability. Despite this, we consider it is still worth investigating the effect of $S$ for the sampling process for completeness
(see evaluation in  Sec.~\ref{sec: exp}).
As such, our investigation of preconditioning matrix is expanded from the invertible symmetric matrix in form of $MM^{\mathsf{T}}$, to more general cases where preconditioning matrices can be written as $MM^{\mathsf{T}}+S$.

\subsection{Instantiation of preconditioned diffusion sampling}
We summarize our \textbf{preconditioned diffusion sampling} (PDS) method for accelerating the diffusion sampling process in Alg.~\ref{algo: tdas2}.
For generality, we write the original diffusion process as 
\begin{equation}
    \mathbf{x}_t =\mathbf{h}(\mathbf{x}_{t-1},t)+\phi(t)\mathbf{z}_t,
\end{equation}
where $\mathbf{h}(\mathbf{x}_{t-1},t)$ represents the drift term and $\phi(t)$ the function controlling the scale of the noise $\mathbf{z}_t$.
We take the real part whilst dropping the imaginary part generated every step as it can not be utilized by the SGMs. Now we construct the space and frequency preconditioning filter.
Given a target dataset image with distribution $p^\ast$, its space preconditioning filter $A$ is calculated as
\begin{align}\label{eq:space_filter}
     A(c,w,h) = \log\big(\mathbb{E}_{\mathbf{x}\sim p^\ast(\mathbf{x})}\left [\mathbf{x}(c,w,h)\right])+1\big),
\end{align}
where $1\leq c \leq C,1\leq w \leq W,1\leq h \leq H$ are the channel, width and height dimensions of image. In practice, we also normalize $A$ for a stability
\begin{align}\label{eq:space_filter_norm}
     A(c,w,h) = \frac{A(c,w,h)}{\max A(c,w,h)}.
\end{align}
There are two approaches for calculating the filter $R$. 
The first approach is to utilize the statistics of the dataset. 
Specifically, we first define the frequency statistics given a specific image dataset that we are aimed to synthesize as
\begin{align}\label{eq:stats1}
    R(c,w,h) = \log\big(\mathbb{E}_{\mathbf{x}\sim p^\ast(\mathbf{x})}\left [ F[\mathbf{x}]\odot \overline{F[\mathbf{x}]} \right](c,w,h)+1\big)
\end{align}
where $F$ is  Discrete Fourier Transform, $\odot$ is the element-wise multiplication.
We then set
\begin{equation}\label{eq:stats2}
   R(c,w,h) = \frac{1}{\alpha}(\frac{R(c,w,h)}{\max R(c,w,h) }+\alpha-1),
\end{equation}
where $\alpha$ is the normalization parameter. 
This allows us to adaptively scale the frequency coordinates according to the specific amplitudes. 
Empirically, $200$ images randomly sampled from the dataset is enough for estimating this statistics, therefore this involves marginal extra computation.
We observe that this approach works well for accelerating NCSN++~\cite{song2020score}, but has less effects on accelerating NCSN~\cite{song2019generative} and NCSNv2~\cite{DBLP:conf/nips/0011E20}. The possible reason is that these two models are not sophisticated enough as NCSN++ to utilize the delicate information from the frequency statistics.
To address this issue, we propose the second approach which constructs the filter $R$ simply using two parameters as follows
\begin{align}\label{eq: freq_mask}
R(c,h,w) = 
\left\{\begin{matrix}
 1&,& \text{if}\;\; (h-0.5H)^2+(w-0.5 W)^2 \leq 2 r^2\\
\lambda&,&\text{otherwise}
\end{matrix}\right.,
\end{align}
where $C$ is the channel number, $H$ is the height, and $W$ is the width of an image.
$1\leq c \leq C$, $1\leq h \leq H$ and $1 \leq w \leq W$.
The parameter $\lambda > 0 $ specifies the ratio for shrinking or amplifying the coordinates located out of the circle $\{(h-0.5H)^2+(w-0.5 W)^2 \leq 2 r^2 \}$, selected according to the failure behaviour of the vanilla SGM. The radial range of the filter is controlled by $r$. 
An example of $R$ is given in Fig.~\ref{fig:freq_mask}. This method works well on accelerating NCSN~\cite{song2019generative} and NCSNv2~\cite{DBLP:conf/nips/0011E20}.

\begin{figure}[h]
\centering
    \centerline{\includegraphics[scale=0.23]{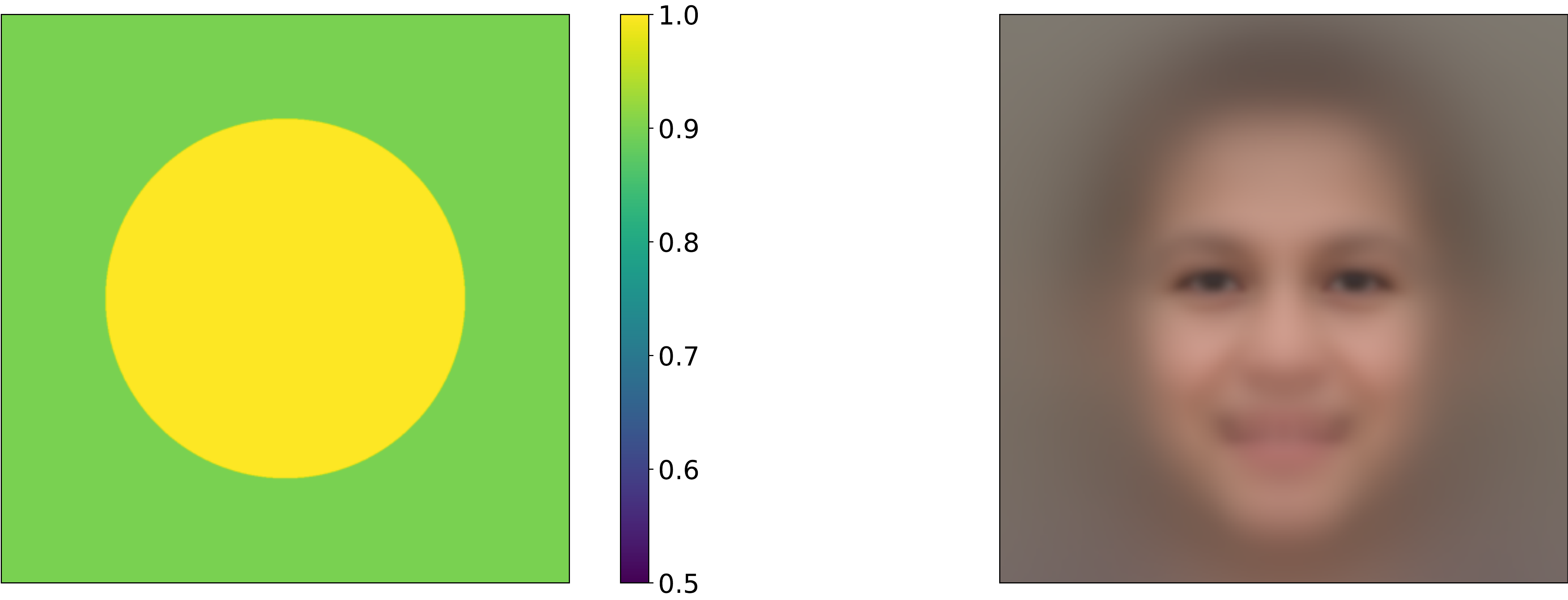}}
    \caption{Examples of (\textbf{Left}) frequency preconditioning $R$ ($(r,\lambda) = (0.2H,0.9)$) and (\textbf{Right}) mean of FFHQ~\cite{karras2019style} dataset used for constructing space preconditioning $A$ used in proposed preconditioning operator $M$ (Eq. \eqref{eq:precondition_op}). 
    }
    \label{fig:freq_mask}
\end{figure}

\begin{algorithm}[tb]
  \caption{Preconditioned diffusion sampling}
  \label{algo: tdas2}
\begin{algorithmic}
  \State {\bfseries Input:} 
  The frequency $R$ and space $A$ preconditioning operators, the target sampling iterations $T$; \\ 
  \State {\bfseries Diffusion process:}
  \State Drawing an initial sample $\mathbf{x}_0 \sim \mathcal{N}(0,I_{C\times H \times W})$
  \For{$t=1$ {\bfseries to} $T$} 
      \State Drawing a noise $\mathbf{w}_t\sim \mathcal{N}(0,I_{C\times H \times W})$
      \State {\em Applying PDS}: $\bm{\eta}_t \leftarrow F^{-1}[F[ \mathbf{w}_t  \bullet A] \bullet R ]$ \Comment{$\bullet$ means element-wise division}
        \State Calculating the drift term $\mathbf{d}_t \leftarrow \mathbf{h}(\mathbf{x}_{t-1},t,\epsilon_t)$
        \State {\em Applying PDS}: $\mathbf{d}_t \leftarrow F^{-1}[F[ F^{-1}[F[\mathbf{d}_t] \bullet R] \bullet A^2] \bullet R ]$
        \State Calculating the solenoidal term $S_t \leftarrow S\bigtriangledown_{\mathbf{x}}\log  p^{\ast}(\mathbf{x}_{t-1})$
        \State Diffusion $\mathbf{x}_{t} \leftarrow Re[\mathbf{d}_t+S_t+\phi(t)\eta_t]$ \Comment{$Re[\cdot]$ means taking the real part}
  \EndFor
  \\ 
  \State {\bfseries Output:} $\mathbf{x}_T$
\end{algorithmic}
\end{algorithm}

\begin{remark}
For the computational complexity of PDS, the major overhead is from FFT and its inverse that only have the complexity of $O(CHW(\log H+\log W))$~\cite{brigham1988fast}, which is neglectable compared to the whole diffusion complexity.
\end{remark}
\section{Experiments}\label{sec: exp}
In our experiments, the objective is to show how off-the-shelf SGMs can be accelerated significantly with the assistance of the proposed PDS whilst keeping the image synthesis quality,
without model retraining. See \ref{sec:para_sm} for the detailed parameter settings for all the experiments.

\paragraph{\bf Datasets.}  
For image synthesis, we use MNIST, CIFAR-10~\cite{krizhevsky2009learning}, LSUN (the tower, bedroom and church classes)~\cite{yu2015lsun}, and FFHQ~\cite{karras2019style} datasets.
Note, for all these datasets, the image height and width are identical, i.e., $H=W$.

\paragraph{\bf Baselines.}
For evaluating the model agnostic property of our PDS, we test three recent SGMs including NCSN~\cite{song2019generative}, NCSNv2~\cite{DBLP:conf/nips/0011E20} and NCSN++~\cite{song2020score}.

\paragraph{\bf Implementation.}
We use the public released codebases of NCSN, NCSNv2 and NCSN++.
For facilitating the comparisons, we follow the same preprocessing as~\cite{song2019generative,DBLP:conf/nips/0011E20,song2020score}.
We conduct all the following experiments with PyTorch on NVIDIA RTX 3090 GPUs.

\paragraph{\bf Experiments on MNIST.}
We use NCSN~\cite{song2019generative} as the SGM for the simplest digital image generation ($28\times 28$). 
The results are shown in Fig.~\ref{fig: mnist}. 
We observe that when reducing the sampling iterations from 1000 to {\em 20} for acceleration, the original sampling method tends to generate images that lack the digital structure (see the left part of Fig.~\ref{fig: mnist}). This suggests us to enlarge a band of frequency part of the diffusion process.
Therefore, we set $(r,\lambda)=(0.2H,1.6)$. 
It is observed that our PDS can produce digital images with the fine digital structure well preserved under the acceleration rate.
\begin{figure}[h]
    \begin{minipage}{0.495\linewidth}
        \centering
        \includegraphics[scale=0.22]{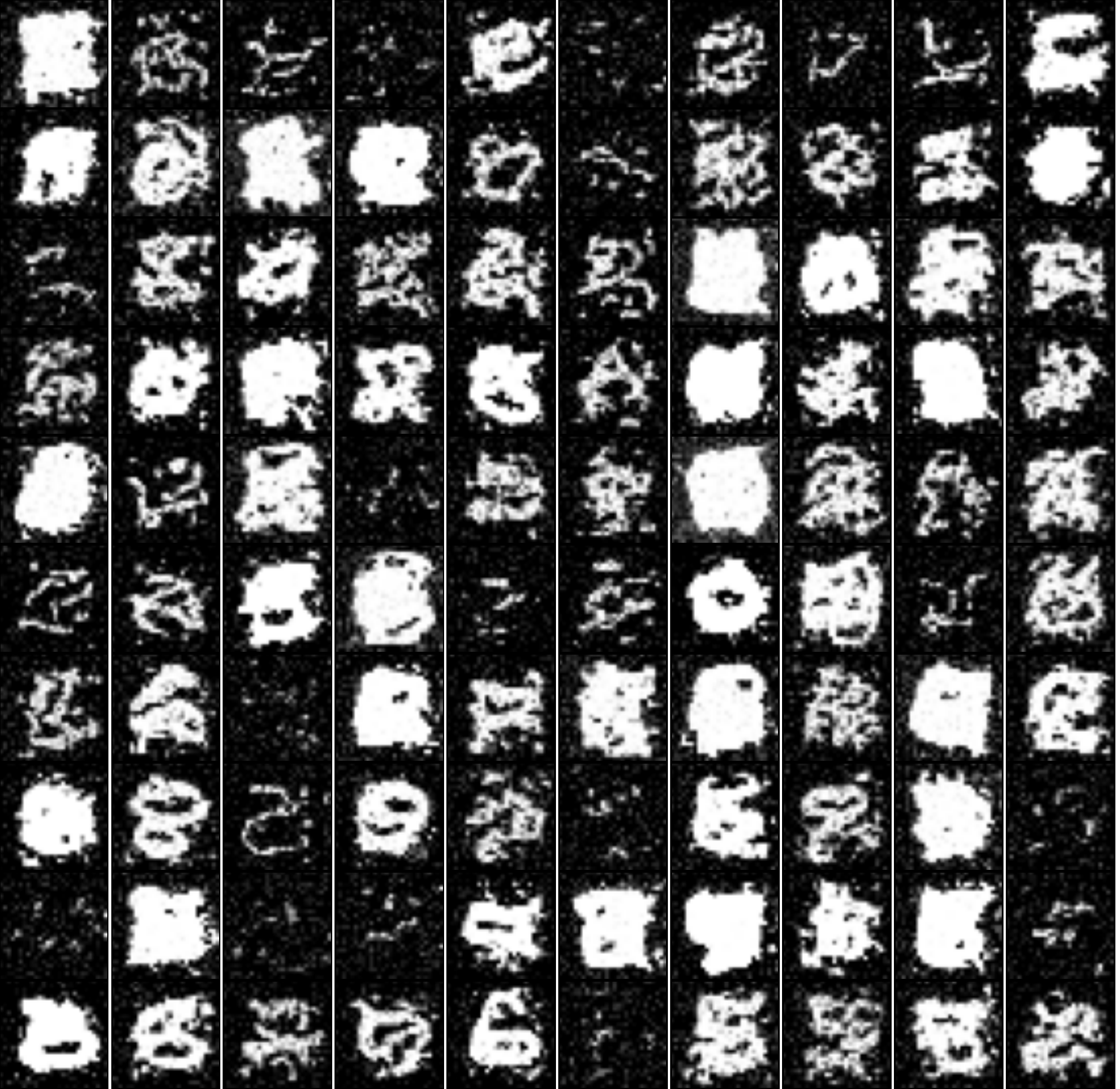}
    \end{minipage}
    \begin{minipage}{0.495\linewidth}
        \centering
        \includegraphics[scale=0.22]{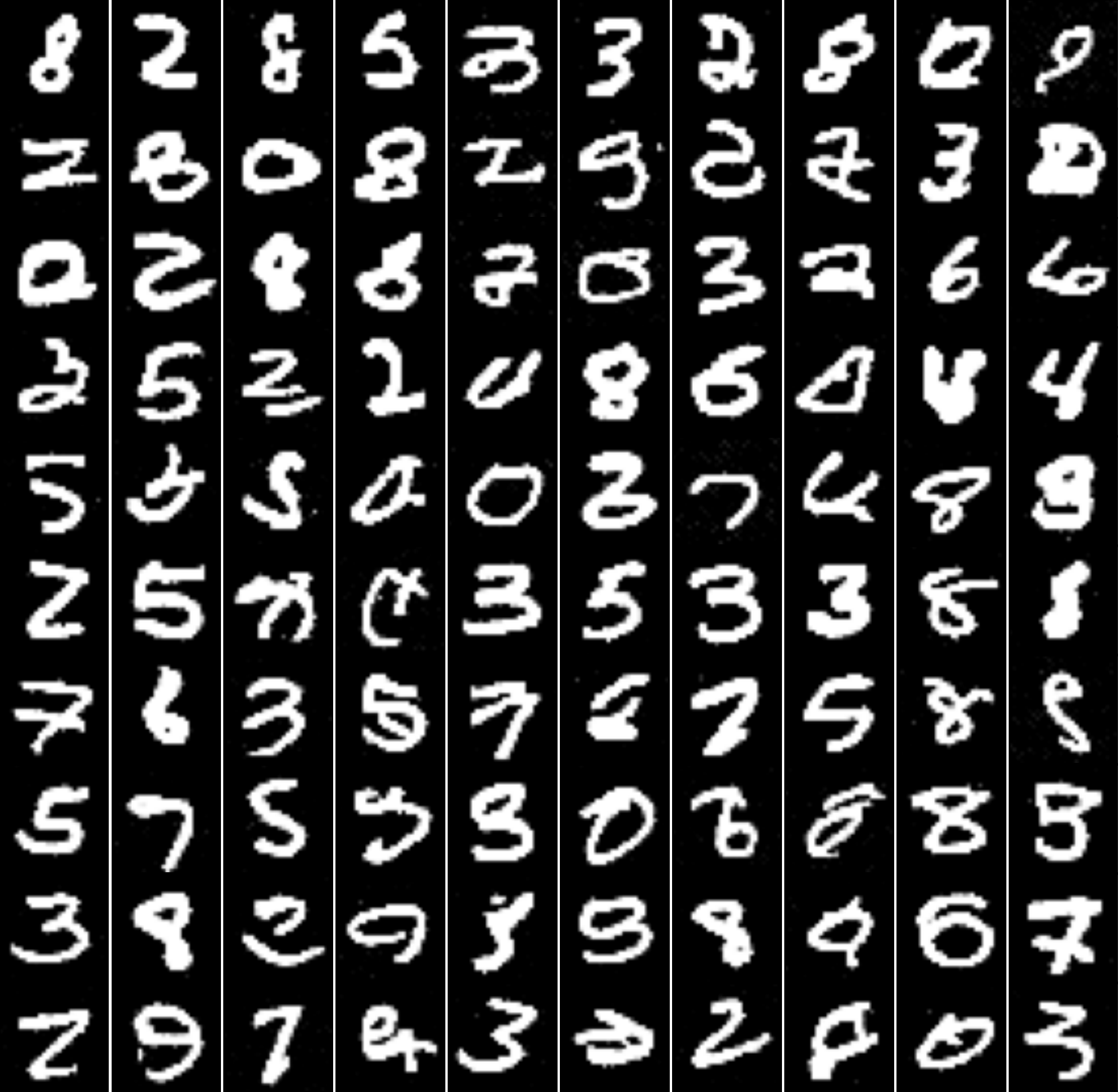}
    \end{minipage}
        \caption{Sampling using NCSN~\cite{song2019generative} on MNIST ($28 \times 28$). \textbf{Left}: Results by the original sampling method with $20$ sampling iterations. \textbf{Right}: Results by our PDS with $20$ sampling iterations. More samples in supplementart material.
        }
        \label{fig: mnist}
\end{figure}

\paragraph{\bf Experiments on CIFAR-10.}
Compared to DDPMs, SGMs have much worse performance when the number of sample iterations is relatively small. Our PDS can greatly alleviate this issue as shown in Table.~\ref{tab:cifar}, where we evaluate NCSN++ for generating CIFAR-10 ($32\times 32$) by FID~\cite{heusel2017gans} score. We compare PDS with DDIM~\cite{song2020denoising} and the Analytic-DDIM~\cite{bao2022analytic}, two representative DDPMs. It is observed that NCSN++ with PDS achieves the best FID scores under different acceleration cases. We apply filter $R$ described by Eq.~\eqref{eq:stats1}.

\begin{table}[h]
\centering
\setlength{\tabcolsep}{1.4mm}{
\caption{FID scores of vanilla NCSN++~\cite{song2020score}, NCSN++ with PDS, DDIM~\cite{song2020denoising},  and Analytic-DDIM~\cite{bao2022analytic} under different iterations on CIFAR-10.}\label{tab:cifar}
\begin{tabular}{ccccc}
\toprule[1.5pt]
$T$ & NCSN++ & DDIM  & Analytic-DDIM  & NCSN++ W/ PDS \\ \midrule
100        & 29.39  & 6.08  & 3.55     & \textbf{3.26}         \\
200        & 4.35  & 4.02   & 3.39   & \textbf{2.61}          \\
\bottomrule[1.5pt]
\end{tabular}}
\end{table}

\paragraph{\bf Experiments on LSUN~\cite{yu2015lsun}.} 
We first evaluate NCSNv2~\cite{DBLP:conf/nips/0011E20} to generate church images at a resolution of $96\times 96$ and tower at a resolution of $128\times 128$.
For both classes, when accelerated by reducing the iterations from original $3258$ to {\em 108} for tower and from original $3152$ to {\em 156} for church, we observe that the original sampling method tends to generate images {\em without sufficient detailed appearance}, similar as the situation on MNIST. Therefore, we also encourage the frequency part of the diffusion process that responsible for the details. The results are displayed in Fig.~\ref{fig:church_ncsnv2}. It is evident that PDS can still generate rich fine details, even when the diffusion process is accelerated up to $20\sim30$ times.

Further, we evaluate NCSN++ \cite{song2020score} to generate bedroom and church images at a resolution of $256\times 256$. 
In this case, we instead observe that the original sampling method tends to generate images {\em with overwhelming noises} once accelerated (left of Fig.~\ref{fig: lsun_ncsnpp}). 
We hence set filter $R$ using Eq.~\eqref{eq:stats1} to regulate the frequency part of the diffusion process.
As demonstrated in Fig.~\ref{fig: lsun_ncsnpp}, our PDS is able to prevent the output images from being ruined by heavy noises.
All these results suggest the ability of our PDS in regulating the different frequency components in the diffusion process of prior SGMs.

\paragraph{\bf Experiments on FFHQ~\cite{karras2019style}.} 
We use NCSN++~\cite{song2020score} to generate high-resolution facial images at a resolution of $1024\times 1024$. Similar as on LSUN, we also find out that when accelerated, the original sampling method is vulnerable with heavy noises and fails to produce recognizable human faces. For example, when reducing the iteration from original $2000$ to $100$, the output images are full of noises and unrecognizable.
Similarly, we address this issue with our PDS with
filter $R$. We also apply the space preconditioning to utilize the structural characteristics shared across the whole dataset.
It is shown in Fig.~\ref{fig:ffhq}, PDS can maintain the image synthesis quality using only as less as {\em 66} iterations.
In summary, all the above experiments indicate that our method is highly scalable and generalizable across different visual content, SGMs, and acceleration rates.

\begin{figure}[h]
    \begin{minipage}{0.233\linewidth}
        \centering
        \includegraphics[scale=0.25]{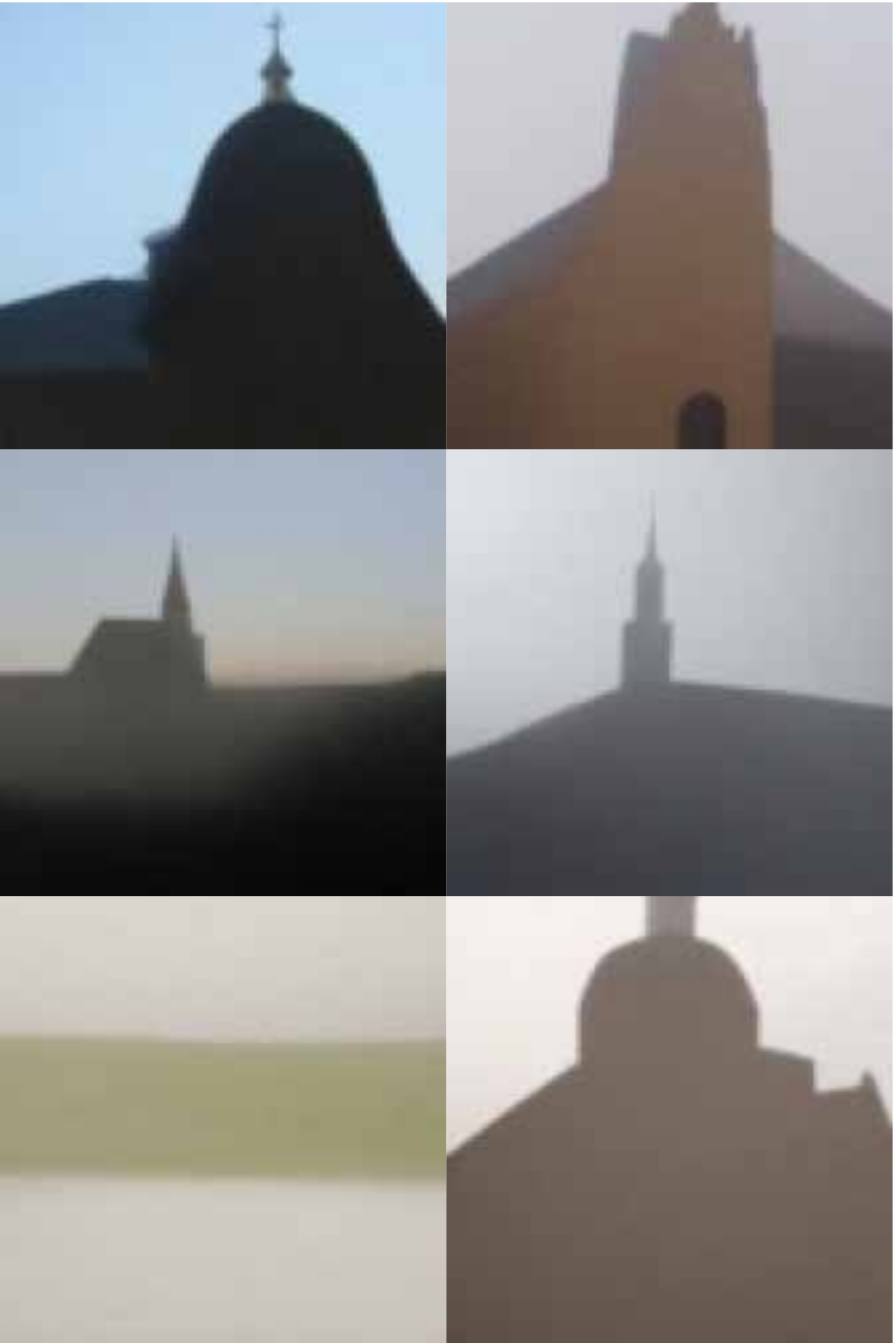}
    \end{minipage}
    \begin{minipage}{0.233\linewidth}
    \centering
    \includegraphics[scale=0.25]{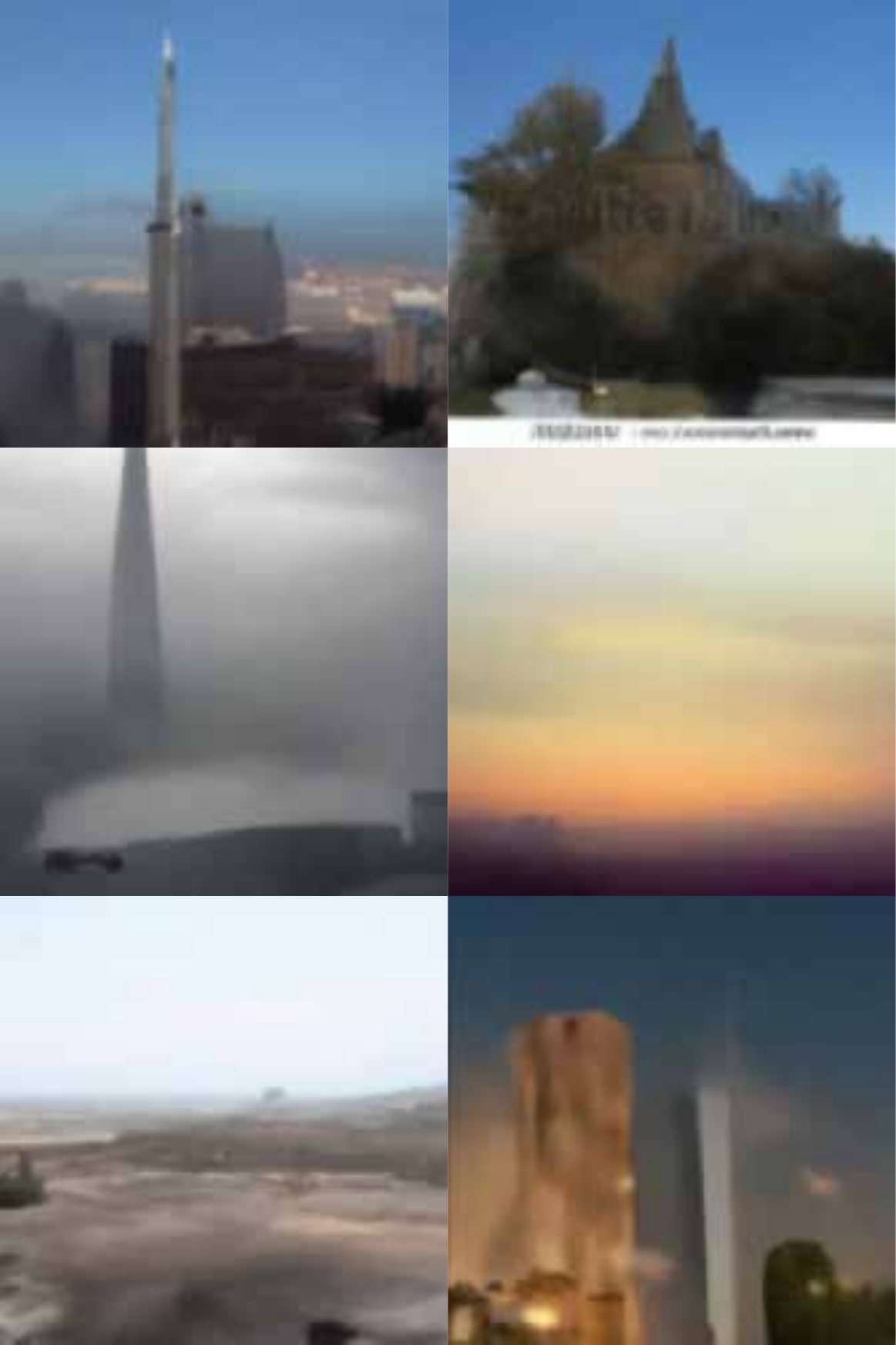}
    \end{minipage}
    \begin{minipage}{0.233\linewidth}
        \centering
        \includegraphics[scale=0.25]{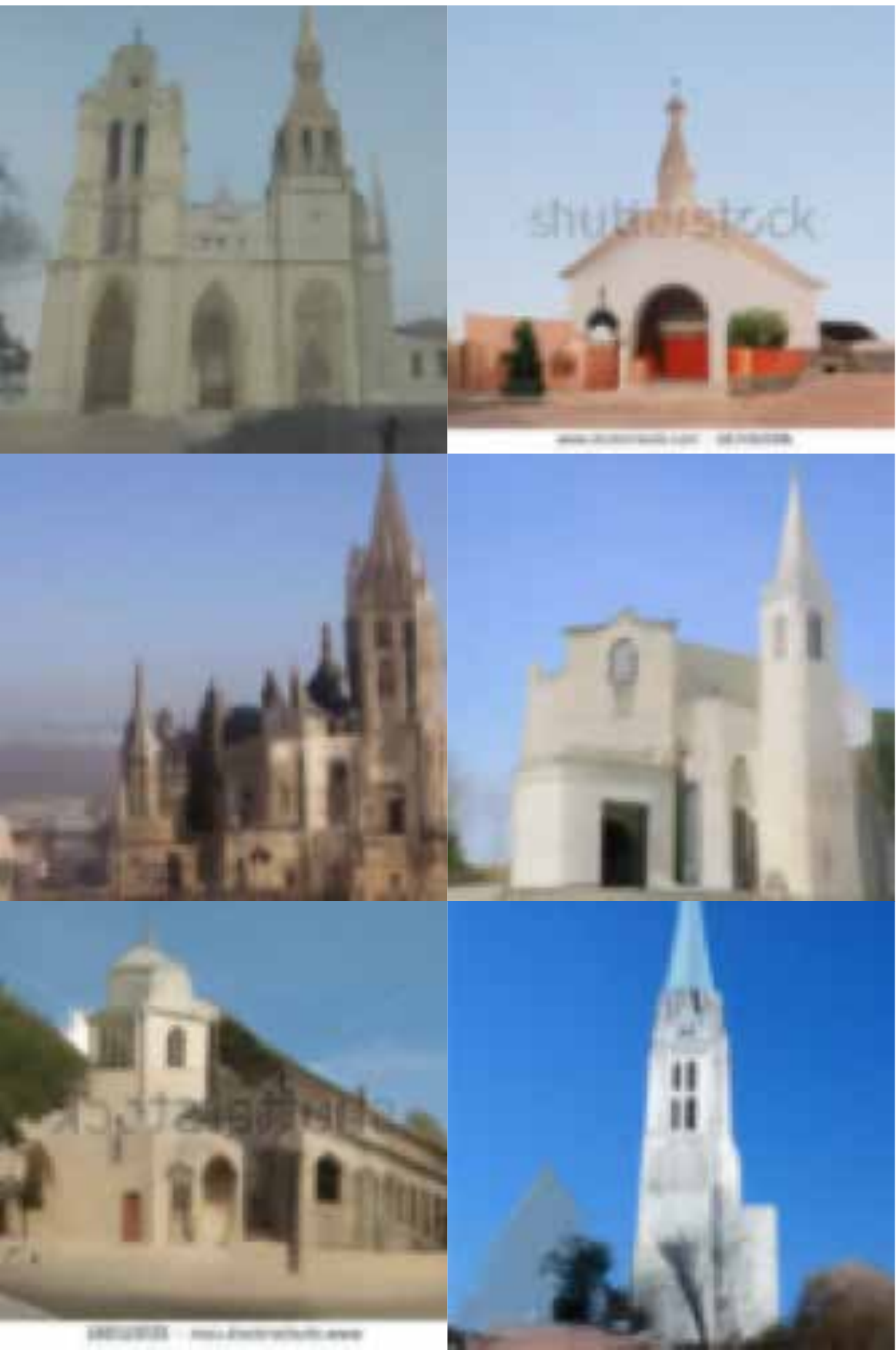}
    \end{minipage}
    \begin{minipage}{0.233\linewidth}
        \centering
        \includegraphics[scale=0.25]{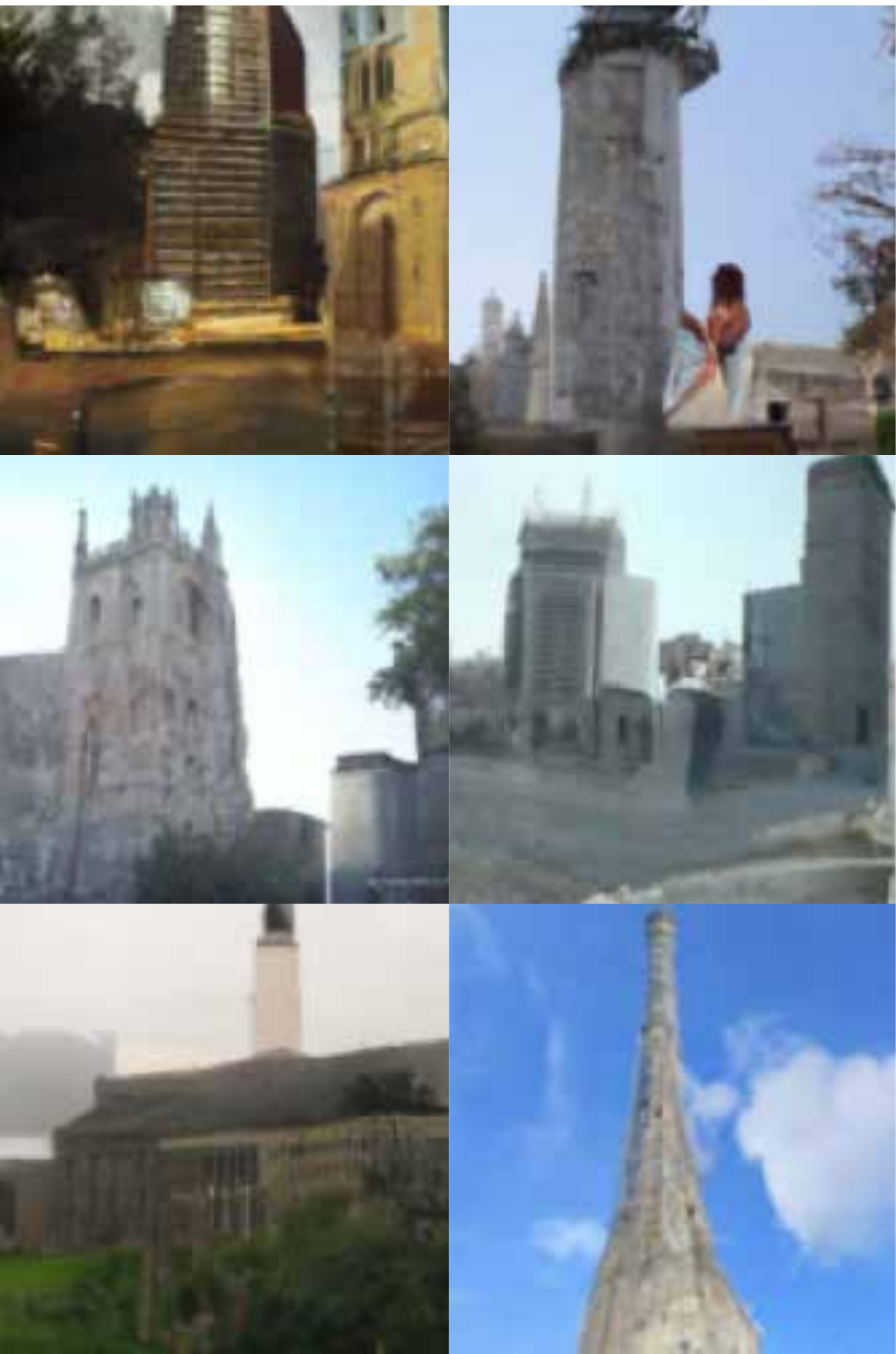}
    \end{minipage}
        \caption{Sampling using NCSNv2~\cite{DBLP:conf/nips/0011E20} on LSUN (church $96 \times 96$ and tower $128 \times 128$).
        \textbf{Left}: The original sampling method with $156$ iterations for church and $108$ iterations for tower. \textbf{Right}: PDS sampling method with $156$ iterations for church and $108$ iterations for tower. More samples in \ref{sec:example_sm}.}
        \label{fig:church_ncsnv2}
\end{figure}
\begin{figure}[h]
    \begin{minipage}{0.233\linewidth}
        \centering
        \includegraphics[scale=0.25]{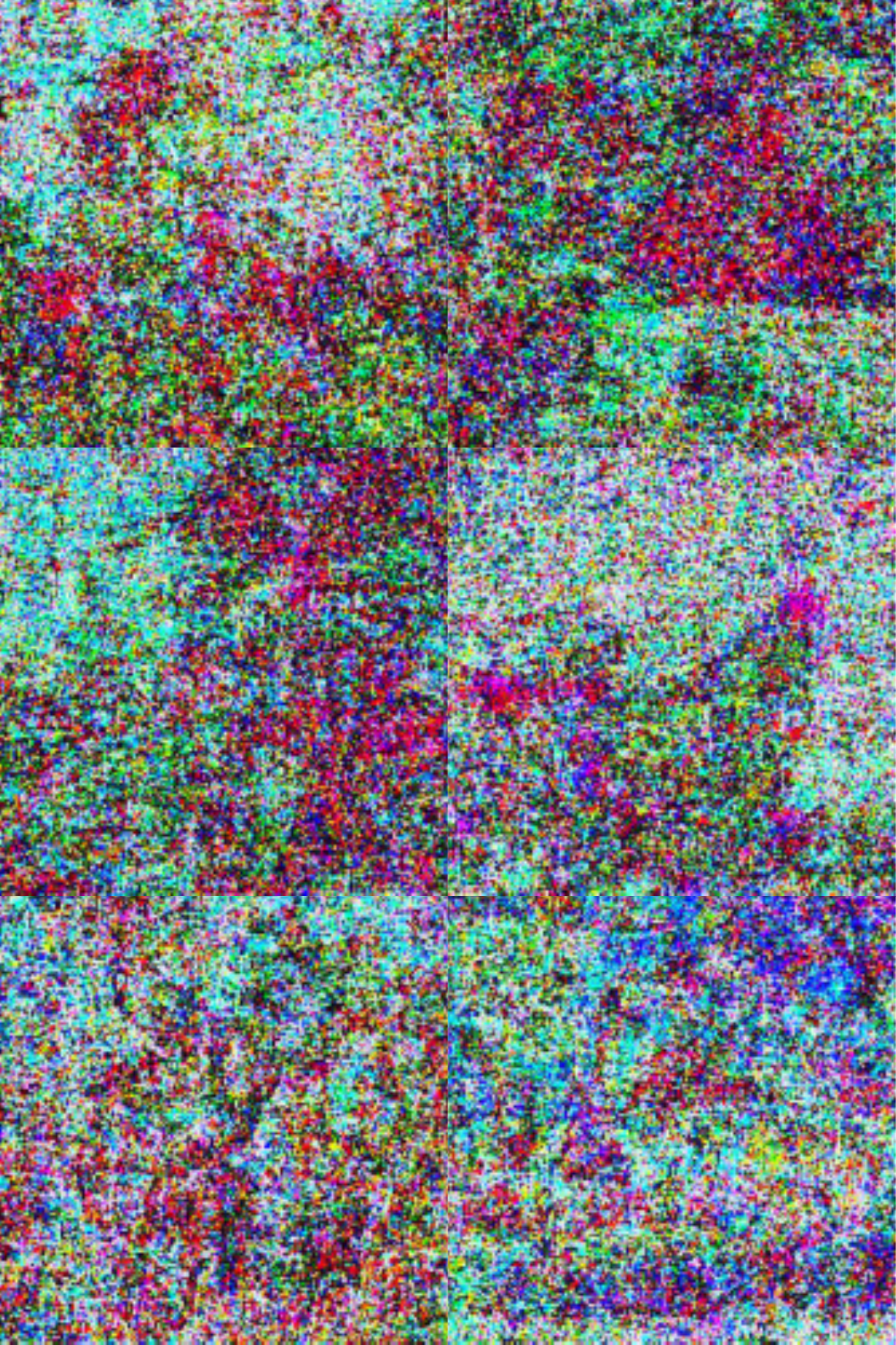}
    \end{minipage}
    \begin{minipage}{0.233\linewidth}
    \centering
    \includegraphics[scale=0.25]{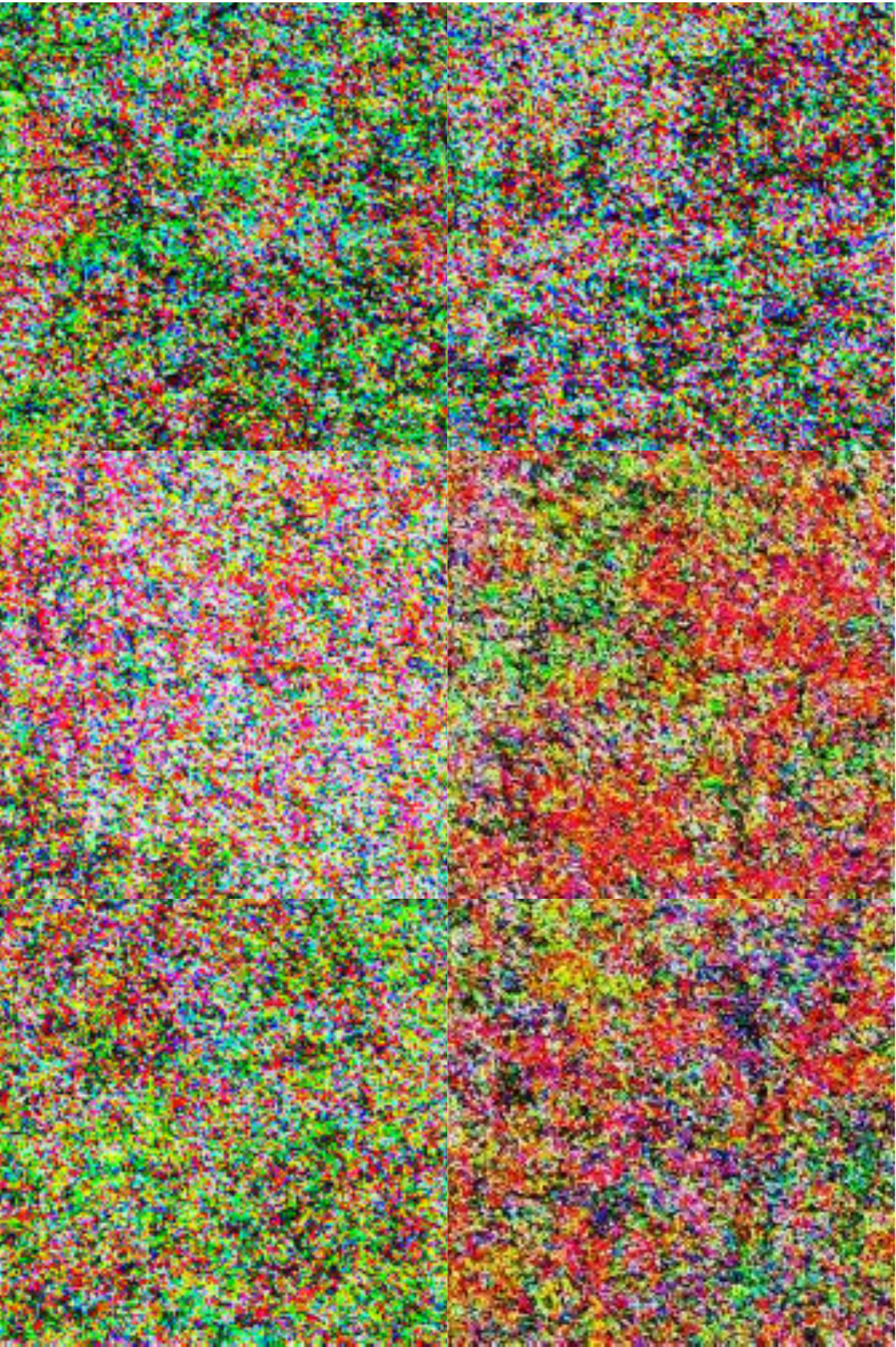}
    \end{minipage}
    \begin{minipage}{0.233\linewidth}
        \centering
        \includegraphics[scale=0.25]{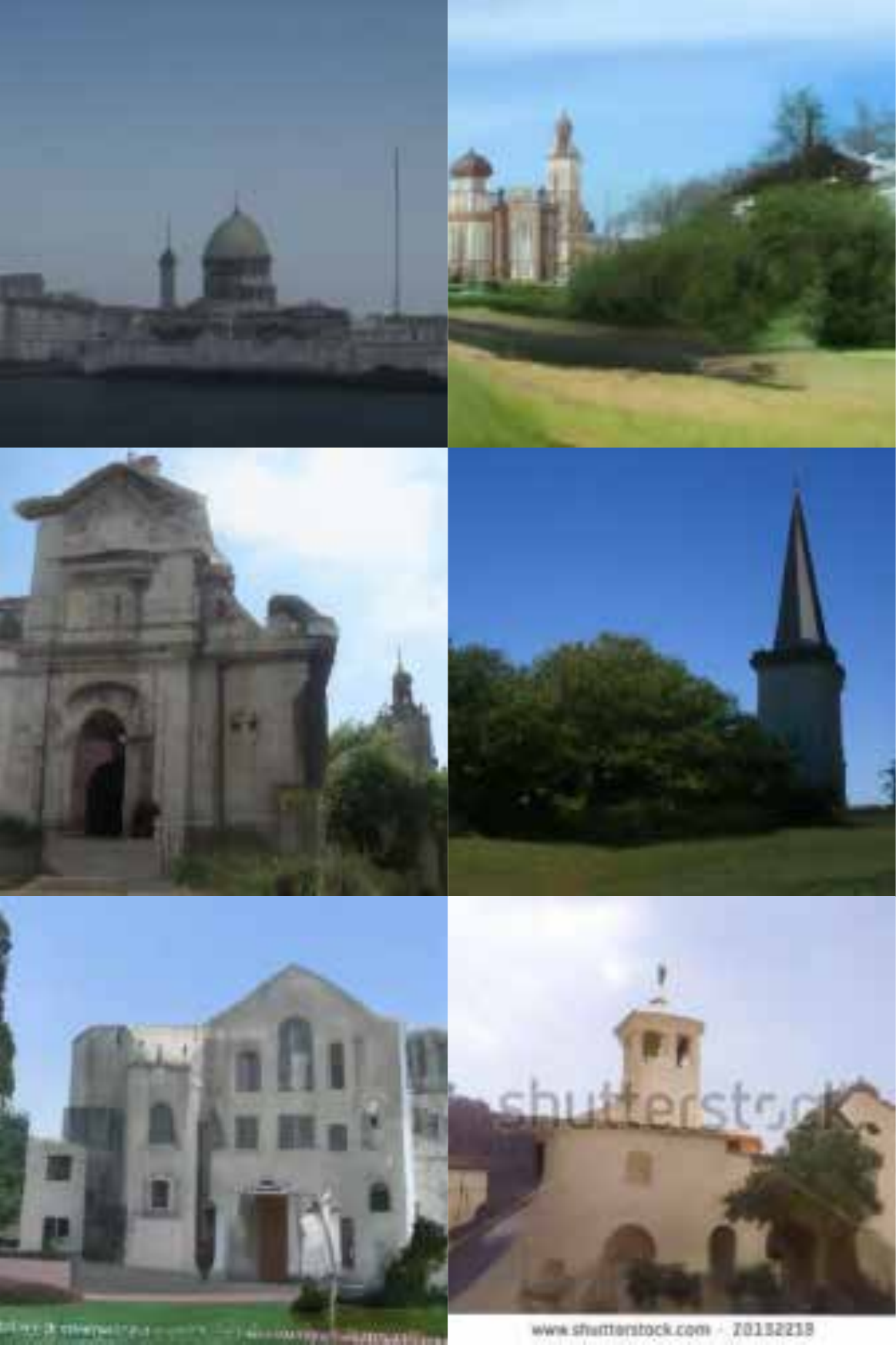}
    \end{minipage}
    \begin{minipage}{0.233\linewidth}
        \centering
        \includegraphics[scale=0.25]{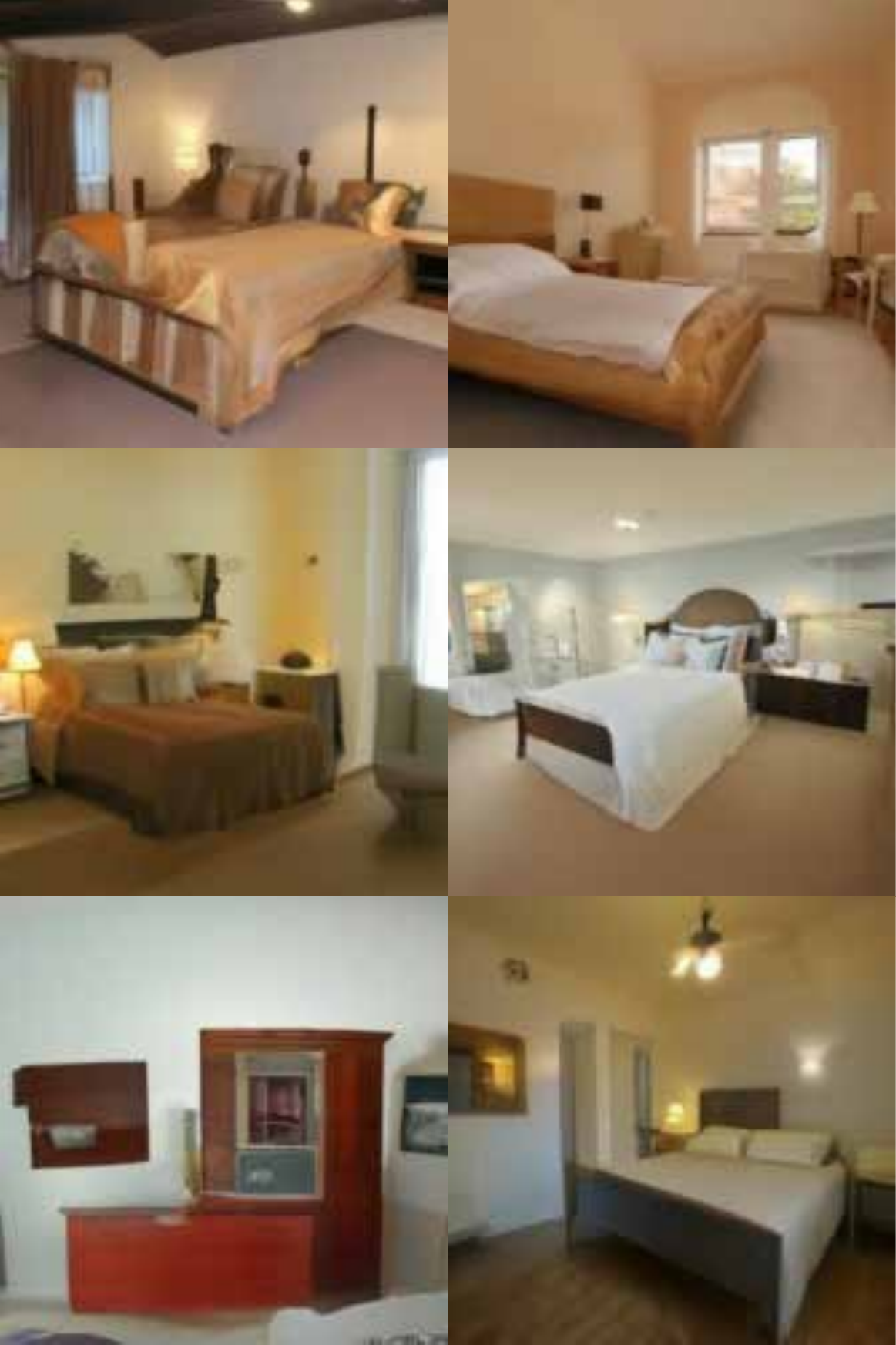}
    \end{minipage}
        \caption{Sampling using NCSN++~\cite{song2020score} on LSUN (church and bedroom) ($256 \times 256$). \textbf{Left}: The original sampling method with $166$ sampling iterations. \textbf{Right}: Our PDS sampling method with $166$ sampling iterations. More examples in \ref{sec:example_sm}.}
        \label{fig: lsun_ncsnpp}
\end{figure}

\paragraph{\bf Evaluation on running speed.}
Apart from the quality evaluation on image synthesis as above, we further compare the running speed between the vanilla and our PDS using NCSN++~\cite{song2020score}.
In this test, we use one NVIDIA RTX 3090 GPU.
We track the average wall-clock time of generating a batch of 8 images.
As shown in Table \ref{tab: time}, our PDS can significantly reduce the running time, particularly for high-resolution image generation on the FFHQ dataset. 

\begin{table}
\begin{center}
\caption{Evaluating the wall-clock time of generating a batch of $8$ images.
{\em SGM:} NCSN++~\cite{song2020score}.
{\em Time unit:} Seconds.}
\label{tab: time}
\setlength{\tabcolsep}{7mm}{
\begin{tabular}{l|ll}
\toprule[1.5pt]
                                  \multicolumn{1}{c}{Dataset} & \multicolumn{1}{c}{LSUN}  & \multicolumn{1}{c}{FFHQ} \\ \midrule
\multicolumn{1}{c}{Vanilla}      & \multicolumn{1}{r}{1173} & \multicolumn{1}{r}{2030}     \\
\multicolumn{1}{c}{\bf PDS}          & \multicolumn{1}{r}{\textbf{90}}  & \multicolumn{1}{r}{\textbf{71}}      \\ \midrule
\multicolumn{1}{c}{\em Speedup times} & \multicolumn{1}{r}{\textbf{13}}    & \multicolumn{1}{r}{\textbf{29}}      \\ \bottomrule[1.5pt]
\end{tabular}}
\end{center}
\vspace{-4ex}
\end{table}

\paragraph{\bf Parameter analysis.}
\label{sec:param_analysis}
We investigate the effect of PDS's two parameters $r$ and $\lambda$ in Eq.~\eqref{eq: freq_mask}. We use NCSN++~\cite{song2020score} with the sampling iterations $T= 166$ on LSUN (bedroom). 
It is observed in Fig.~\ref{fig:para} that there exists a large good-performing range for each parameter. If $\lambda$ is too high or $r$ is too low, PDS will degrade to the vanilla sampling method, yielding corrupted images;
Instead, if $\lambda$ is too low or $r$ is too high, which means over-suppressing high-frequency signals in this case, 
pale images with fewer shape details will be generated.
For NCSN++~\cite{song2020score}, since we directly use the statistics information to construct $R$, there is no need to worry about selecting $r$ and $\lambda$.
\begin{figure}[h]
    \begin{minipage}{1\linewidth}
        \textcolor{white}{--------------}$\lambda = 0.85$\textcolor{white}{---------}$\lambda = 0.9$\textcolor{white}{----------}$\lambda = 0.91$\textcolor{white}{----------}$\lambda = 0.92$\textcolor{white}{----------}$\lambda = 0.95$
    \end{minipage}
    \rotatebox{90}{\textcolor{white}{-----} $r = 0.25H $\textcolor{white}{-------}$r = 0.2H $\textcolor{white}{-------}$r = 0.15H $} 
    \centerline{\includegraphics[scale=0.29]{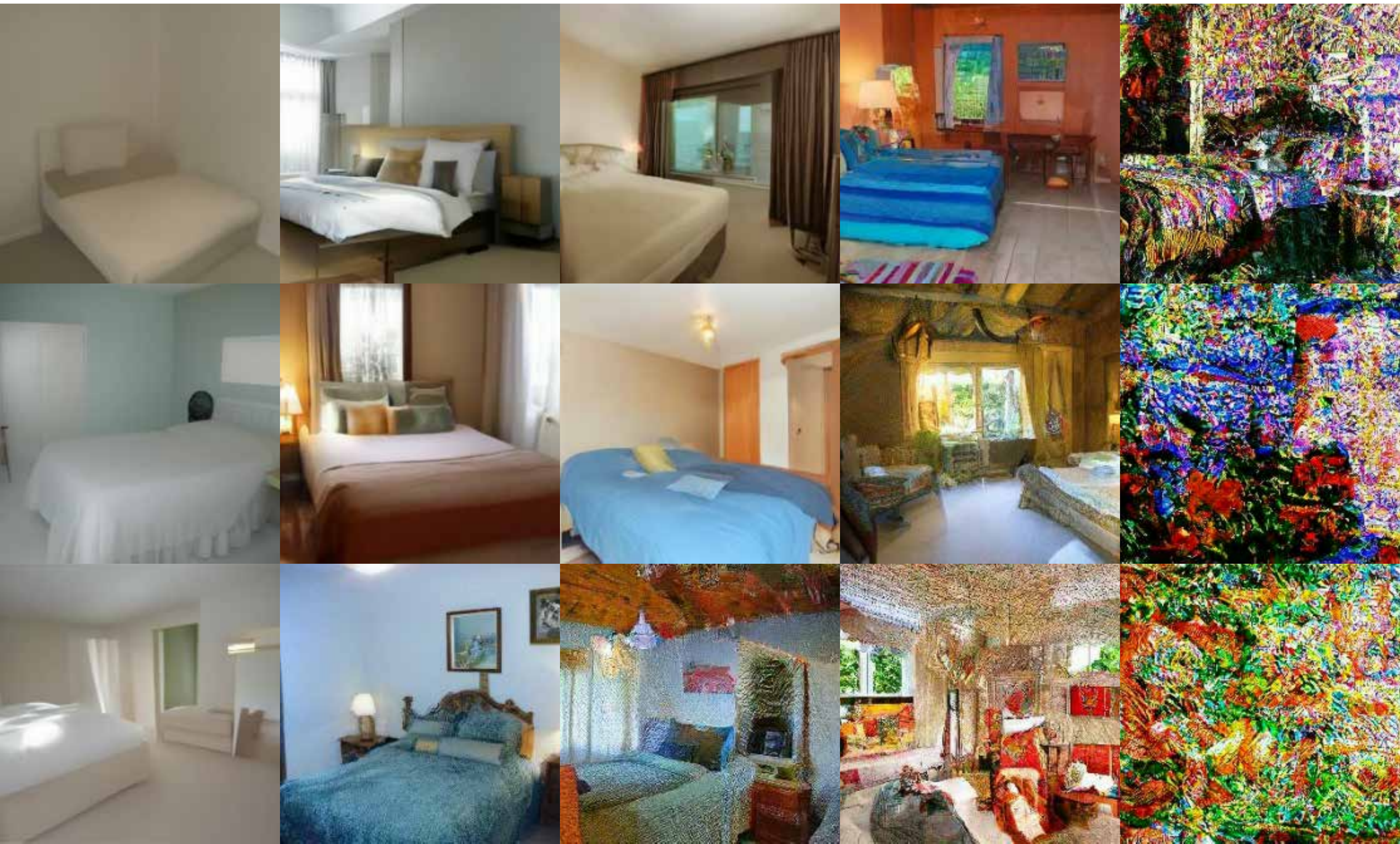}}
    \caption{{\bf Parameter analysis.} Sampling produced by PDS using NCSN++~\cite{song2020score} on LSUN (bedroom) ($256 \times 256$) with $166$ sampling iterations. We set $(r,\lambda)$ to a variety of combination. }
    \label{fig:para}
\end{figure}

\begin{figure}[h]
\centering
    \begin{minipage}{1\linewidth}
        \textcolor{white}{------}\hspace{6ex}$\omega =1 $\hspace{15ex}$\omega = 10$\hspace{15ex}$\omega = 100$\hspace{12ex}$\omega = 1000$
    \end{minipage}
    \begin{minipage}{1\linewidth}
    \centering
    \centerline{\includegraphics[scale=0.195]{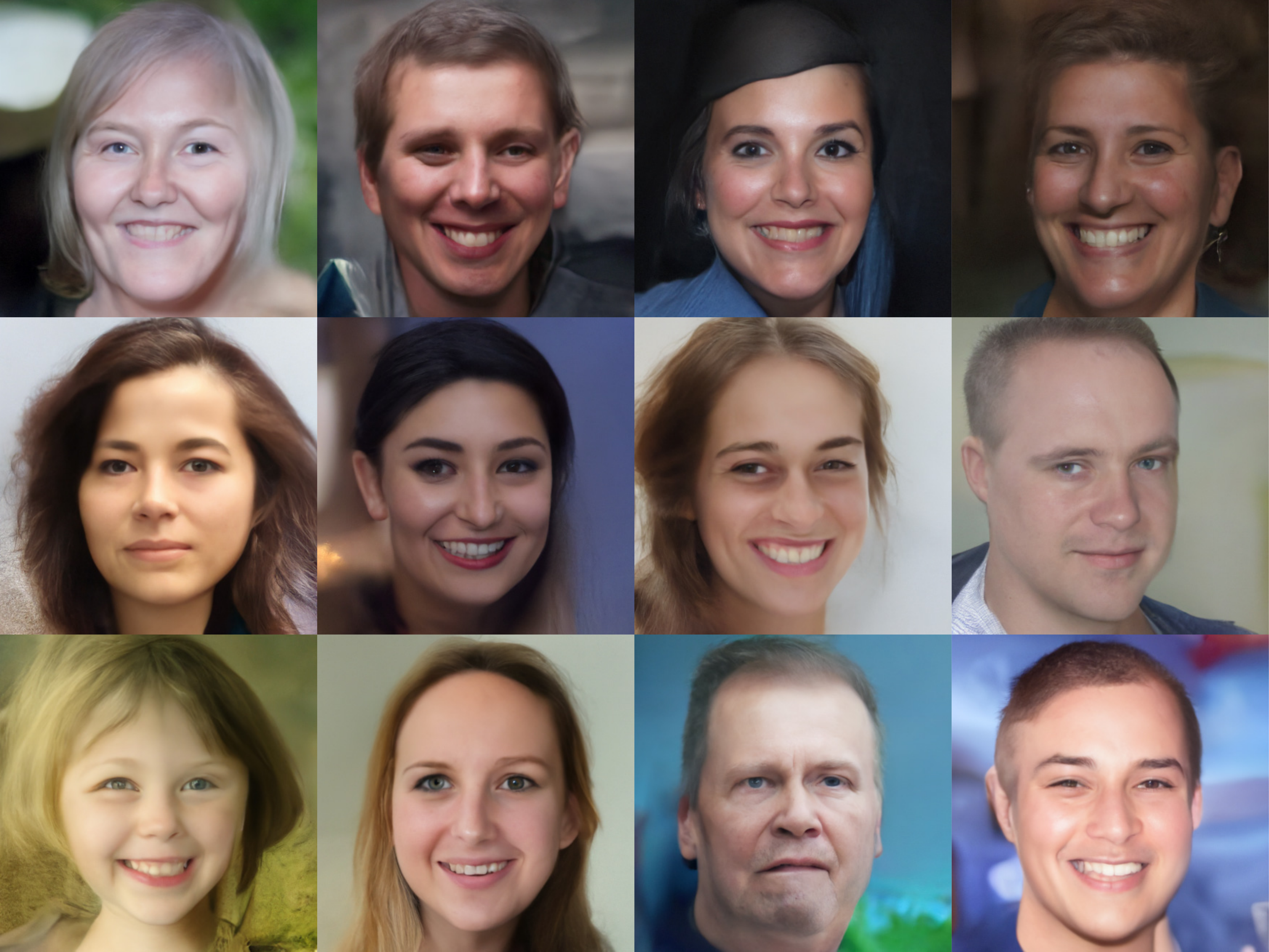}}
    \end{minipage}
        \caption{Samples produced by PDS using NCSN++~\cite{song2020score} on FFHQ ($1024 \times 1024$) with different solenoidal terms (controlled by $\omega$). The sampling iteration is set to $66$. More samples in \ref{sec:solenoidal_sm}. }
        \label{fig: solenoidal}
\end{figure}

\paragraph{\bf Further analysis.}
In this section, we study the effect of the solenoidal term $S\bigtriangledown_{\mathbf{x}} \log p^{\ast}(\mathbf{x})$~\footnote{For NCSN++~\cite{song2020score}, we use $\bigtriangledown_{\mathbf{x}} \log p_t(\mathbf{x})$, where $p_t$ is the distribution function of $\mathbf{x}$ at $t$, since $\bigtriangledown_{\mathbf{x}} \log p^{\ast}(\mathbf{x})$ is inaccessible in NCSN++.} to the diffusion process. As proved in Thm.~\ref{thm: unchange2}, as long as $S$ is skew-symmetric, it will not change the steady-state distribution of the original process. To verify this claim experimentally, we further generalize the original process as
\begin{align}
    d\mathbf{x} = \frac{\epsilon^2}{2}(M^{-1}M^{\mathbf{-T}}+\omega S)\bigtriangledown_{\mathbf{x}}\log p^{\ast}(\mathbf{x})dt + \epsilon M^{-1}d\mathbf{w},
\end{align}
where $\omega$ is the parameter that controls the scale of $S$. In Fig.~\ref{fig: solenoidal}, we set $S[\cdot] = Re[F[\cdot] - F^{\mathsf{T}}[\cdot]]$ which is obviously skew-symmetric. 
We change the scale of $\omega$ from $1$ to $1000$ for evaluating its impact on the output samples. It is observed that $\omega$ does not affect the quality of output images.
This verifies that $S$ does not change the steady-state distribution of the original diffusion process. 
Additionally, we perform similar tests with different iterations and other different skew-symmetric operator $S$.
We still observe no obvious acceleration effect from the solenoidal term (see \ref{sec:solenoidal_sm}).

\section{Limitations}
In general, there are several parameters in the preconditioning matrix of PDS need to be determined. A further study is needed to enable PDS find the best parameter settings automatically.
Although DDPMs are a variant of SGMs, we find PDS can not directly used on DDPMs, since the diffusion process of DDPMs is not a Langevin dynamics. Nevertheless, we find that it is possible to rewrite this diffusion process to imitate the structure of Langevin dynamics, then use PDS for acceleration. 
We leave it for future study.

\section{Conclusion}
In this work, we have proposed a novel preconditioned diffusion sampling (PDS) method for accelerating off-the-shelf score-based generative models (SGMs), without model retraining. 
Considering the diffusion process
as a Metropolis adjusted Langevin algorithm, we reveal that existing sampling suffers from ill-conditioned curvature.
To solve this, we reformulate the diffusion process with matrix preconditioning whilst preserving its steady-state distribution (i.e., the target distribution), leading to our PDS solution.
Experimentally, we show that PDS can significantly accelerate existing state-of-the-art SGMs while maintaining the generation quality.

\section*{Acknowledgments}
This work was supported in part by 
National Natural Science Foundation of China (Grant No. 6210020439),
Lingang Laboratory (Grant No. LG-QS-202202-07),
Natural Science Foundation of Shanghai (Grant No. 22ZR1407500),
Shanghai Municipal Science and Technology Major Project (Grant No. 2018SHZDZX01 and 2021SHZDZX0103),
Science and Technology Innovation 2030 - Brain Science and Brain-Inspired Intelligence Project (Grant No. 2021ZD0200204).


\bibliographystyle{splncs04}
\bibliography{reference}

\appendix

\section{Appendix}
\subsection{Preconditioning a diffusion process in the frequency domain}\label{sec:freq_sm}
In this section, we will prove theoretically why we can {\em directly} regulate the frequency distribution of a diffusion process through the preconditioning strategy, and why it is necessary to do so.

We first show that a diffusion process can be directly transformed to another space (e.g., the frequency domain) via an orthogonal transform. To minimize ambiguity, we denote $p^{\ast}(\mathbf{x})$ as $p^{\ast}_{\mathbf{x}}(\mathbf{x})$.
\begin{theorem}
The Langevin dynamics
\begin{align}\label{eq: Langevin_1}
    d\mathbf{x} = \frac{\epsilon^2}{2}\bigtriangledown_{\mathbf{x}}\log p_{\mathbf{x}}^{\ast}(\mathbf{x})dt + \epsilon d\mathbf{w}
\end{align}
can be rewritten as
\begin{align}\label{eq: Langevin_2}
    d\tilde{\mathbf{x}} = \frac{\epsilon^2}{2}\bigtriangledown_{\tilde{\mathbf{x}}}\log p_{\tilde{\mathbf{x}}}^{\ast}(\tilde{\mathbf{x}})dt + \epsilon d\mathbf{w},
\end{align}
where $\tilde{\mathbf{x}} :s= B\mathbf{x}$, given $B$ is an orthogonal transform. 
\end{theorem}
\begin{proof}
Multiplying $B$ on both sides of Eq.~\eqref{eq: Langevin_1}, we have:
\begin{align}
       d\tilde{\mathbf{x}} = \frac{\epsilon^2}{2}B\bigtriangledown_{\mathbf{x}}\log p_{\mathbf{x}}^{\ast}(\mathbf{x})dt + \epsilon Bd\mathbf{w}.
\end{align}
We have $Bd\mathbf{w}=d\mathbf{w}$ by the rotational invariance of the standard Wiener process. Now we only need to verify
\begin{align}\label{eq:need}
    B\bigtriangledown_{\mathbf{x}}\log p_{\mathbf{x}}^{\ast}(\mathbf{x}) = \bigtriangledown_{\tilde{\mathbf{x}}}\log p_{\tilde{\mathbf{x}}}^{\ast}(\tilde{\mathbf{x}}).
\end{align}
Given two $d$-dimensional random vectors $\mathbf{x},\mathbf{y}\in\mathbb{R}^d$ with their respective differentiable density functions $ p_{\mathbf{x}}$ and $ p_{\mathbf{y}}$, if $g(\mathbf{x}) = \mathbf{y}$, where $g\in \mathbb{R}^d \rightarrow \mathbb{R}^d$ is an invertible differentiable transformation, we have
\begin{align}
    p_{\mathbf{y}}(\mathbf{y}) = p_{\mathbf{x}}(g^{-1}(\mathbf{y}))\left |\det\left[ \frac{d g^{-1}(\mathbf{y})}{d\mathbf{y}}\right] \right |.
\end{align}
Therefore, 
\begin{align}\label{eq:trans}
    \bigtriangledown_{\mathbf{x}}\log p_{\mathbf{x}}^{\ast}(\mathbf{x}) =
    \bigtriangledown_{\mathbf{x}}[\log p_{\tilde{\mathbf{x}}}^{\ast}(\tilde{\mathbf{x}}) - \log\left |\det\left[ B^T\right] \right |] = \bigtriangledown_{\mathbf{x}}\log p_{\tilde{\mathbf{x}}}^{\ast}(\tilde{\mathbf{x}}).
\end{align}
Using the chain rule of the calculus, we have
\begin{align}\label{eq:chain}
\bigtriangledown_{\mathbf{x}}\log p_{\tilde{\mathbf{x}}}^{\ast}(\tilde{\mathbf{x}})
   =B^T\bigtriangledown_{\tilde{\mathbf{x}}}\log p_{\tilde{\mathbf{x}}}^{\ast}(\tilde{\mathbf{x}}).
\end{align}
Combining Eq.~\eqref{eq:trans} and Eq.~\eqref{eq:chain}, we have
\begin{align}
    \bigtriangledown_{\mathbf{x}}\log p_{\mathbf{x}}^{\ast}(\mathbf{x})=B^T\bigtriangledown_{\tilde{\mathbf{x}}}\log p_{\tilde{\mathbf{x}}}^{\ast}(\tilde{\mathbf{x}}),
\end{align}
which is equivalent to Eq.~\eqref{eq:need} using the orthogonality of $B$.
\end{proof}
\begin{remark}
The above result is easy to be extended to a more general case where the drift term $\frac{\epsilon^2}{2}\bigtriangledown_{\mathbf{x}}\log p_{\mathbf{x}}^{\ast}(\mathbf{x})$ is replaced by $f(t)\mathbf{x}+\bigtriangledown_{\mathbf{x}}\log q(\mathbf{x},t)$, if $f$ is a scalar function of time and $q(\cdot,t)$ is a distribution function that may vary over time. Therefore, the theorem can be applied generally to all the diffusion processes adopted in NCSN~\cite{song2019generative}, NCSNv2~\cite{DBLP:conf/nips/0011E20}, and NCSN++~\cite{song2020score}.  
\end{remark}
Specially, when we set $B$ as a two-dimensional discrete cosine transform~\cite{ahmed1974discrete,BOVIK200997}, the whole diffusion process can be transformed to the frequency domain without changing its original form. {\em This explains why we can directly implement a preconditioning operator on the original diffusion process to regulate its frequency distribution.}

There exists a general observation that the amplitude of the high-frequency part of a natural image is dramatically lower than that in the low-frequency part~\cite{BOVIK200997}. This means the distribution of natural images exhibits huge gaps in quantity between different coordinates in the frequency domain, causing a severe ill-conditioned issue. {\em This explains the necessity to regulate the frequency distribution of a diffusion process, which is implemented by preconditioning in this paper.}
\subsection{Parameter settings}\label{sec:para_sm}
We provide the parameter settings used in our experiments in Table.~\ref{tab:para_sm2}. For NCSN~\cite{song2019generative} and NCSNv2~\cite{DBLP:conf/nips/0011E20}, we construct the frequency filter $R$ following Eq.~\eqref{eq: freq_mask}. The two parameters $r$ and $\lambda$ used in each dataset is shown in Table.~\ref{tab:para_sm}. For these two models, we do not apply the space preconditioning.

For NCSN++~\cite{song2020score}, we we construct the frequency filter $R$ following Eq.~\eqref{eq:stats1} and q.~\eqref{eq:stats2}. We apply the space preconditioning following Eq.~\eqref{eq:space_filter} and Eq.~\ref{eq:space_filter_norm} for FFHQ dataset, since there is a clear space structure priors (the layout of human faces), and we do not apply the space preconditioning for other datasets.

\begin{table}[h]
\caption{Parameters of PDS used for constructing frequency filter on NCSN~\cite{song2019generative} and NCSNv2~\cite{DBLP:conf/nips/0011E20} following Eq.~\eqref{eq: freq_mask}. }\label{tab:para_sm2}
\vspace{-4ex}
\begin{center}
    \begin{tabular}{ccccccc}
\toprule[1.5pt]
Dataset                           & Resolution                           &  Model                      & Iterations & $r$       & $\lambda$  & use space preconditioning? \\ \midrule
MNIST                             & $28\times 28$                        &  NCSN                       & $20$    & $0.2H$    &  1.6  &\XSolidBrush \\ \midrule
\multirow{3}{*}{LSUN (church) }   & \multirow{3}{*}{$96\times 96$}       &  \multirow{3}{*}{NCSNv2}    & $126$   & $0.2H$    &  1.6  &\XSolidBrush \\ 
                                  &                                      &                             & $157$   & $0.2H$    &  1.6  &\XSolidBrush \\ 
                                  &                                      &                             & $210$   & $0.2H$    &  1.6  &\XSolidBrush \\ \midrule
\multirow{3}{*}{LSUN (tower) }    & \multirow{3}{*}{$128\times 128$}     &  \multirow{3}{*}{NCSNv2}    & $65$    & $0.2H$    &  1.1  &\XSolidBrush \\
                                  &                                      &                             & $81$    & $0.2H$    &  1.1  &\XSolidBrush \\
                                  &                                      &                             & $108$   & $0.2H$    &  1.1  &\XSolidBrush   \\ \bottomrule[1.5pt]
\end{tabular}
\end{center}
\end{table}

\begin{table}[h]
\caption{Parameters of PDS used for constructing frequency filter on NCSN++~\cite{song2020score} following Eq.~\eqref{eq:stats1} and Eq.~\eqref{eq:stats2}. }\label{tab:para_sm}
\vspace{-4ex}
\begin{center}
    \begin{tabular}{ccccc}
\toprule[1.5pt]
Dataset                   & Resolution                    & Iterations & $\alpha$   & use space preconditioning?  \\ \midrule
\multirow{2}{*}{CIFAR-10} & \multirow{2}{*}{$32\times32$} & 100        & 5        &  \XSolidBrush    \\
                          &                               & 200        & 10       &  \XSolidBrush    \\           \midrule
LSUN (bedroom)            & $256\times256$                & 166        & 5        &  \XSolidBrush    \\
LSUN (church)             & $256\times256$                & 166        & 5        &  \XSolidBrush    \\
FFHQ                      & $1024\times1024$              & 66         & 5        &  \checkmark       \\  \bottomrule[1.5pt]
\end{tabular}
\end{center}

\end{table}

\begin{table}[h]
\begin{center}
\caption{Quantitative evaluation with the Clean-FID (Fr\'echet Inception Distance) metric~\cite{parmaraliased} for accelerated diffusion process. We generate $50k$ images for each method.}\label{tab:fid}
\setlength{\tabcolsep}{2.3mm}{
\begin{tabular}{c|cc|ccc}
\toprule[1.5pt]
Model  & \multicolumn{2}{c|}{NCSNv2~\cite{DBLP:conf/nips/0011E20}}& \multicolumn{3}{c}{NCSN++~\cite{song2020score}} \\ \midrule
Dataset  & \multicolumn{1}{c}{LSUN} & \multicolumn{1}{c|}{LSUN} & \multicolumn{1}{c}{LSUN} & \multicolumn{1}{c}{LSUN} & \multicolumn{1}{c}{FFHQ} \\ \midrule
Class & \multicolumn{1}{c}{Church} & \multicolumn{1}{c|}{Tower} & \multicolumn{1}{c}{Bedroom} & \multicolumn{1}{c}{Church} & \multicolumn{1}{c}{Face} \\ \midrule
Resolution &
\multicolumn{1}{c}{$96\times96$} & \multicolumn{1}{c|}{$128\times128$} & \multicolumn{1}{c}{$256\times256$} & \multicolumn{1}{c}{$256\times256$} & \multicolumn{1}{c}{$1024\times1024$}
\\ \midrule
\multicolumn{1}{l|}{Iterations} & 156                                                                          & 108                                                                          & 166                                                                           & 166                                                                          & 66                                                                           \\
\midrule
Vanilla      & 217.9                                                                       & 67.2                                                                        & 393.7                                                                         & 393.3                                                                        & 463.2                                                                       \\
\textbf{PDS} & \textbf{65.7}                                                                & \textbf{43.8}                                                               & \textbf{16.9}                                                                 & \textbf{15.0}                                                                & \textbf{61.2}                                                                     \\ \bottomrule[1.5pt]
\end{tabular}}
\end{center}
\end{table}

\subsection{More quantitative results}
In this section, we report more quantitative results using Clean-FID (Fr\'echet Inception Distance) metric~\cite{parmaraliased} to verify that our PDS accelerates the vanilla diffusion process while generating images with high quality. 
It is observed in Table~\ref{tab:fid} that the Clean-FID scores of our PDS are all dramatically smaller than those by the original methods in all the cases, consistent with our visualization results (Fig.~\ref{fig: ncsnpp_lsun}-\ref{fig:ffhq_hq}).

\subsection{Solenoidal term analysis}\label{sec:solenoidal_sm}

\begin{figure}[H]
    \vspace{-12ex}
    \centering
    \begin{minipage}{1\linewidth}
        \textcolor{white}{------}\hspace{3ex}$S_1$\hspace{12.5ex}$S_2$\hspace{12.5ex}$S_3$\hspace{12.5ex}$S_4$\hspace{12.5ex}$S_5$\hspace{12.5ex}$S_6$
    \vspace{-24ex}
    \end{minipage}
    \begin{minipage}{1\linewidth}
        \centerline{\includegraphics[scale=0.21]{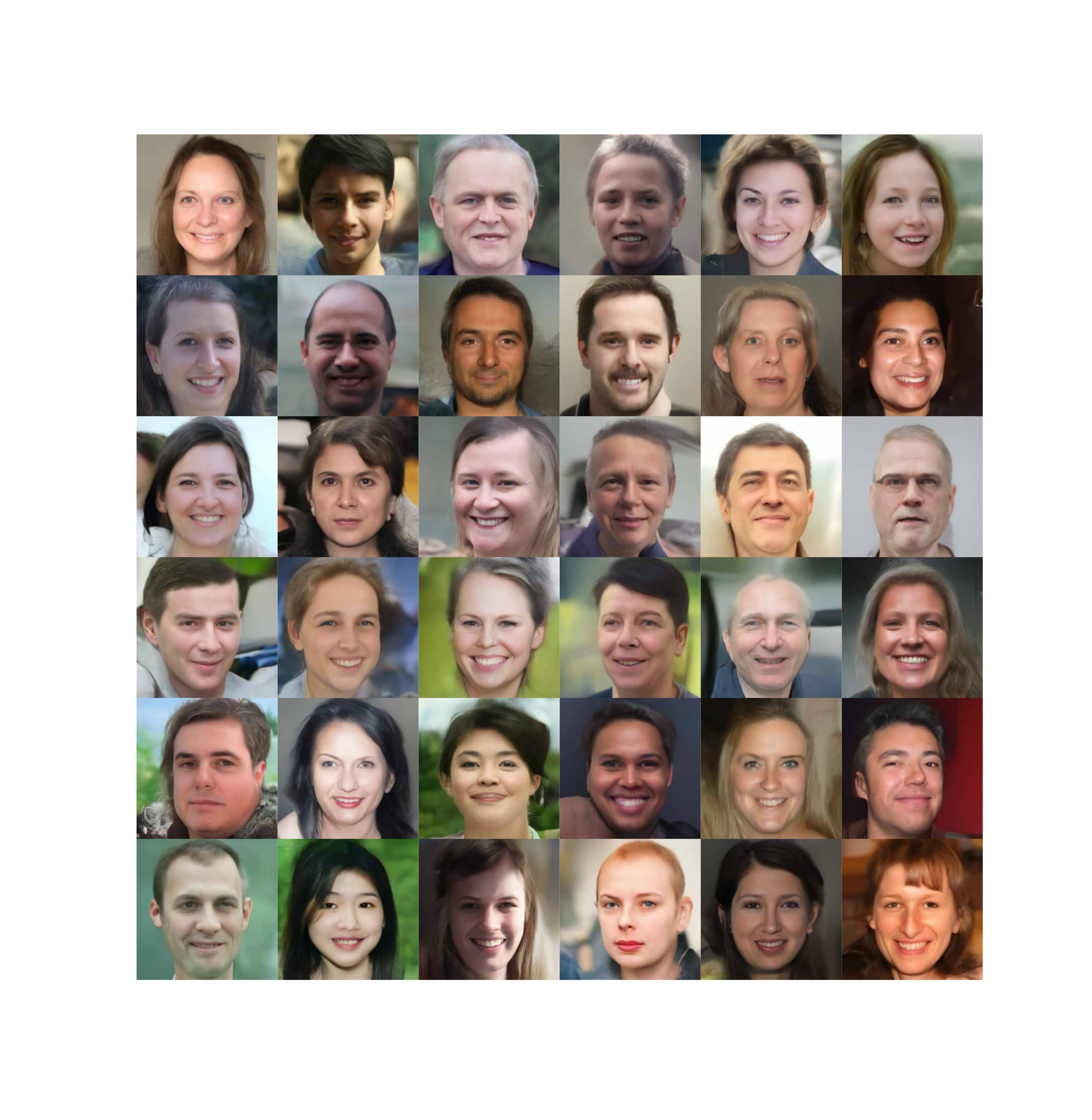}}
    \end{minipage}

    \vspace{-12ex}
        \caption{
        Facial images at a resolution of $1024\times 1024$ generated by NCSN++~\cite{song2020score}
        with our PDS using different solenoidal items.
        Sampling iterations: $66$.
        Dataset: FFHQ~\cite{karras2019style}.
    		}
        \label{fig:ffhq_s_ms}
        \vspace{-4ex}
\end{figure}

In this section, we investigate the effect of the solenoidal term $S\log p^{\ast}(\textbf{x})$ to the diffusion process
\begin{equation}
    d\mathbf{x} = \frac{\epsilon^2}{2}(M^{-1}M^{\mathbf{-T}}+\omega S)\bigtriangledown_{\mathbf{x}}\log p^{\ast}(\mathbf{x})dt + \epsilon M^{-1}d\mathbf{w}
\end{equation}
In Sec. 5 of the main paper, we have shown that using $S = Re[F-F^{\mathsf{T}}]$ has no obvious effect on the diffusion process. Now we study more cases. Denote $P_{m,n}$ as the shift operator that rolls the input image for $m$ places along the height coordinate and rolls the input image for $n$ places along the width coordinate. We then test how the skew-symmetric operators in Eq.~\eqref{eq: so_term} would affect the diffusion process.

\begin{align}\label{eq: so_term}
    &S_1= P_{1,1} - P_{1,1}^{\mathbf{T}}\\ \nonumber
    &S_2=P_{10,10} - P_{10,10}^{\mathbf{T}}\\\nonumber
    &S_3=P_{100,100} - P_{100,100}^{\mathbf{T}}\\\nonumber
    &S_4=Re[F[P_{1,1}- P_{1,1}^{\mathbf{T}}]F^{-1}]\\\nonumber
    &S_5=Re[F[P_{10,10}- P_{10,10}^{\mathbf{T}}]F^{-1}]\\\nonumber
    &S_6=Re[F[P_{100,100}- P_{100,100}^{\mathbf{T}}]F^{-1}].
\end{align}
The sampling results are shown in Fig.~\ref{fig:ffhq_s_ms}, where we set $\omega = 1000$. It is observed that again all these solenoidal terms do not impose an obvious effect on the sampling quality. Additionally, as displayed in Fig.~\ref{fig:ffhq_s_ms_diff}, these solenoidal terms also do not make an obvious effect on acceleration. Nevertheless, we only study the effect of some special cases of the solenoidal terms, which does not mean there are no solenoidal terms that can accelerate the diffusion process, and the search for these solenoidal terms is in a further study.

\begin{figure}[h]
    \vspace{-25ex}
    \centering
    \begin{minipage}{1\linewidth}
         \textcolor{white}{------}\hspace{10.5ex}$T = 66$\hspace{10.5ex}$T = 57$\hspace{10.5ex}$T = 50$\hspace{10.5ex}$T = 44$
    \vspace{-38ex}
    \end{minipage}
    \begin{minipage}{1\linewidth}\hspace{1ex}
          \rotatebox{90}{\textcolor{white}{------}\hspace{22ex}$S=0$\hspace{15ex}$S_1$\hspace{15ex}$S_2$\hspace{15ex}$S_3$\hspace{15ex}$S_4$\hspace{15ex}$S_5$\hspace{15ex}$S_6$} \hspace{-4.5ex}
        \centerline{\includegraphics[scale=0.245]{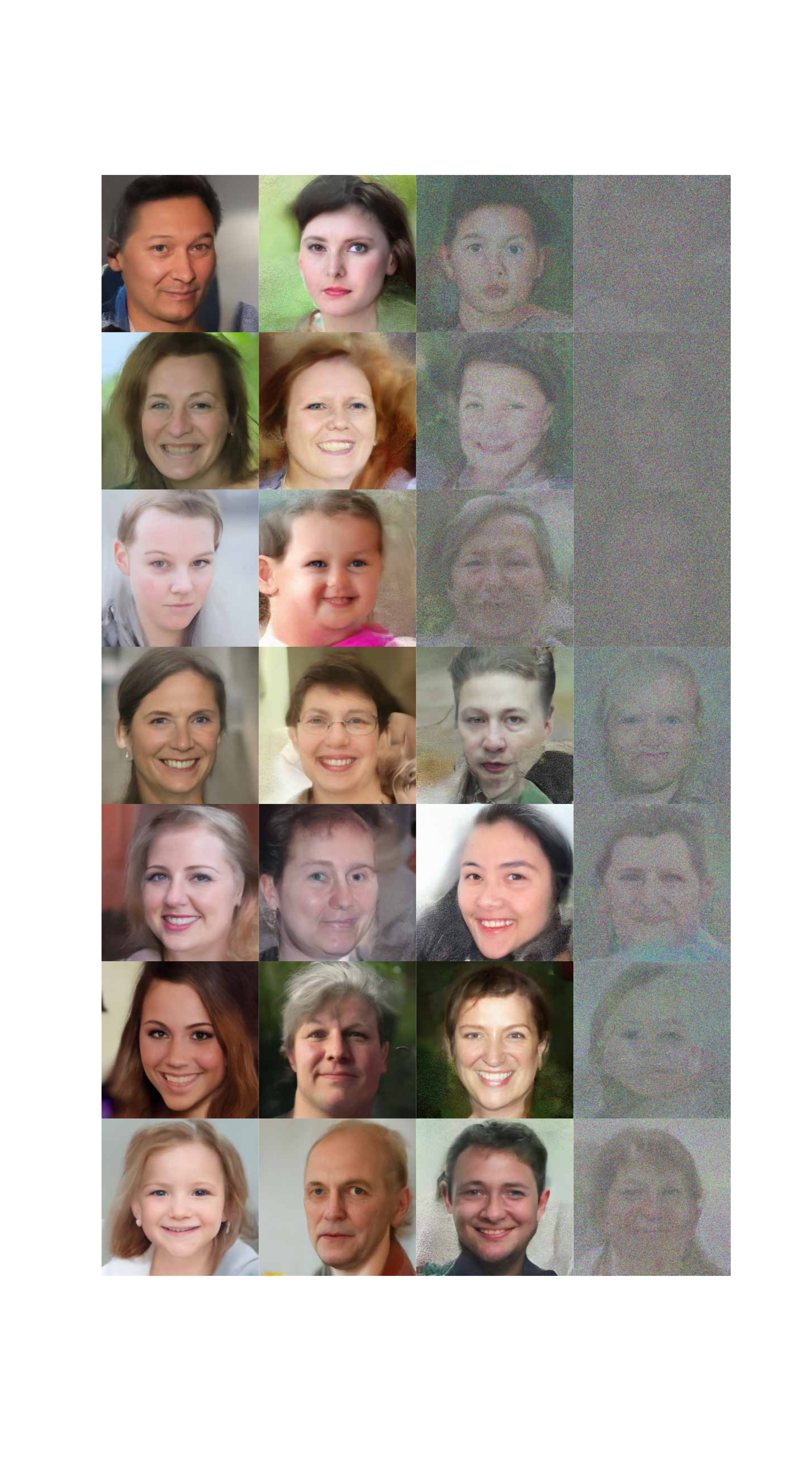}}
    \end{minipage}

    \vspace{-20ex}
        \caption{
        Facial images at a resolution of $1024\times 1024$ generated by NCSN++~\cite{song2020score}
        with our PDS
        under different sampling iterations and different solenoidal items described in Eq.~\eqref{eq: so_term}. We set $R$ following Eq.~\eqref{eq:stats2} and do not apply space preconditioning
        Dataset: FFHQ~\cite{karras2019style}.
        }
        \label{fig:ffhq_s_ms_diff}
        \vspace{-4ex}
\end{figure}
\clearpage
\subsection{More examples}\label{sec:example_sm}

\begin{figure}[ht]
    \vspace{-4ex}
    \centerline{\includegraphics[scale=0.33]{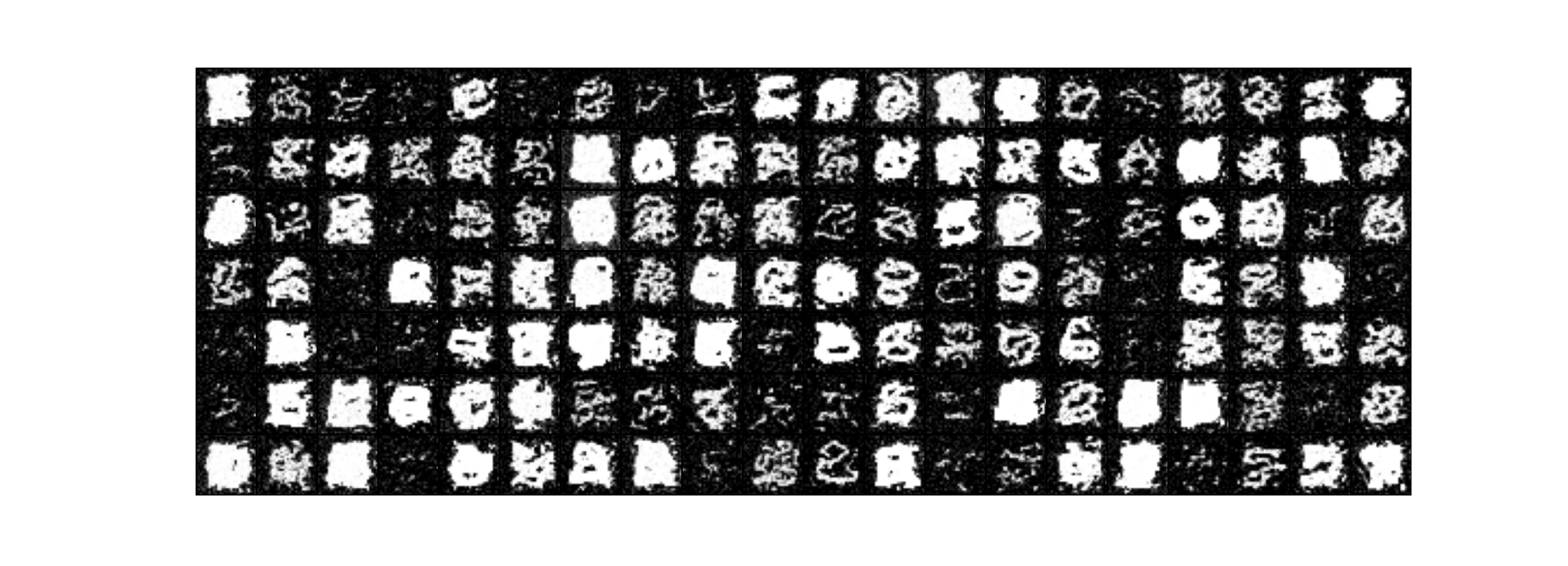}}
         \vspace{-9.8ex}
        \centerline{\includegraphics[scale=0.33]{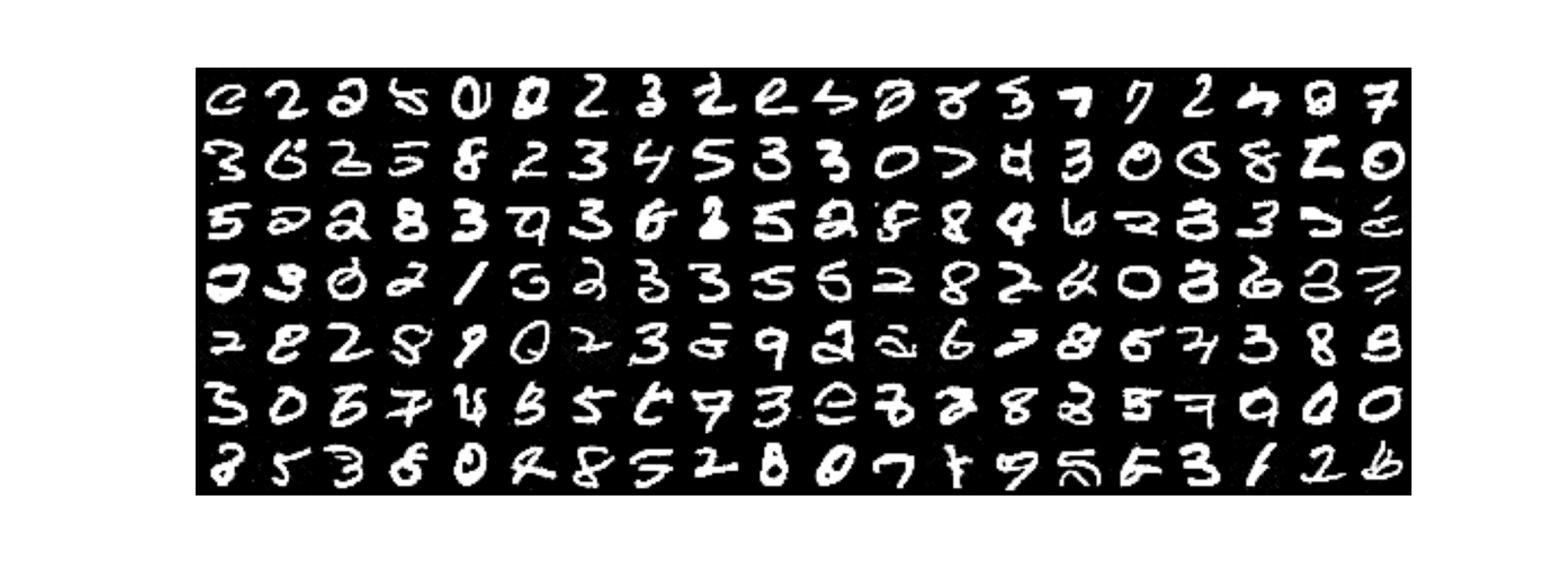}}
    \vspace{-3ex}
        \caption{Sampling using NCSN~\cite{song2019generative} on MNIST ($28 \times 28$). \textbf{Top}: Results by the original sampling method with $20$ sampling iterations. \textbf{Bottom}: Results by our PDS with $20$ sampling iterations, where we set $(r,\lambda)=(0.2H,1.6)$. 
        }      
        \label{fig: mnist_sm}
        \vspace{-6ex}
\end{figure}
\begin{figure}[ht]
    \vspace{-17ex}
    \begin{minipage}{1\linewidth}
        \textcolor{white}{------}\hspace{1ex}$T = 210$\hspace{6ex}$T = 157$\hspace{6ex}$T = 126$\hspace{7ex}$T = 108$\hspace{6ex}$T = 81$\hspace{6ex}$T = 65$
    \vspace{-26ex}
    \end{minipage}
    
    \rotatebox{90}{\textcolor{white}{------}\hspace{-34ex} Ours \hspace{46ex} Baseline~\cite{DBLP:conf/nips/0011E20}} \hspace{-0.5ex}
    \begin{minipage}{0.47\linewidth}
        \centerline{\includegraphics[scale=0.24]{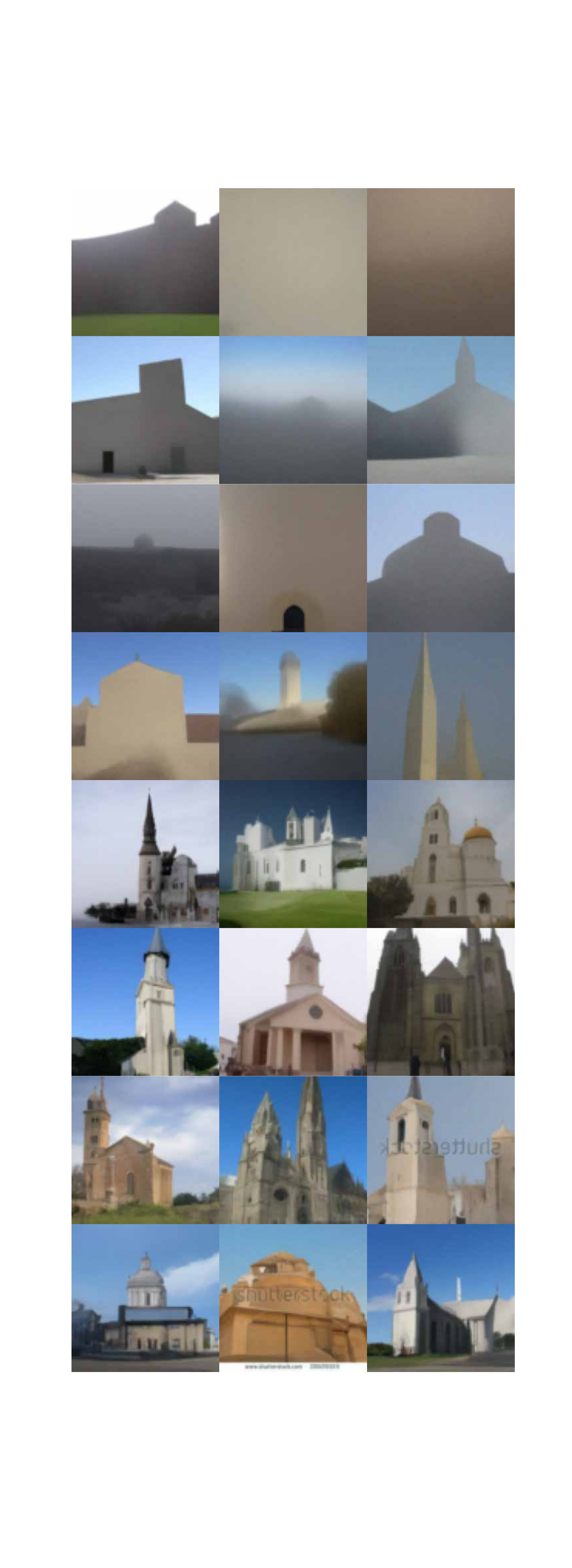}}
    \end{minipage}
    \begin{minipage}{0.47\linewidth}
            \centerline{\includegraphics[scale=0.24]{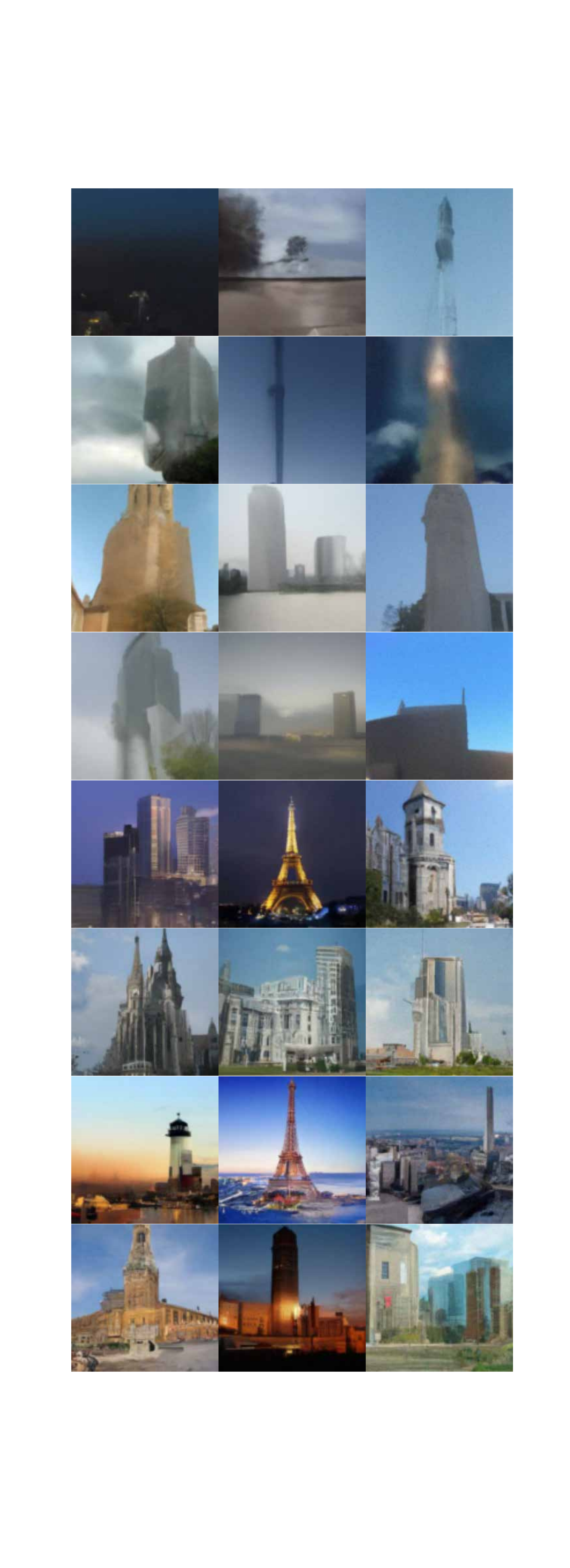}}
    \end{minipage}
    \vspace{-16ex}
           \caption{Sampling using NCSNv2~\cite{DBLP:conf/nips/0011E20} on LSUN (church $96\times 96$ and tower $128 \times 128$) under different iteration numbers.}
        \label{fig: ncsnv2_lsun}
        \vspace{-6ex}
\end{figure}
\begin{figure}[ht]
    \vspace{-17ex}
    \begin{minipage}{1\linewidth}
        \textcolor{white}{------}\hspace{1ex}$T = 200$\hspace{5.6ex}$T = 166$\hspace{5.6ex}$T = 142$\hspace{6.6ex}$T = 200$\hspace{5.6ex}$T = 166$\hspace{5.6ex}$T = 142$
    \vspace{-26ex}
    \end{minipage}
    
    \rotatebox{90}{\textcolor{white}{------}\hspace{-34ex} Ours \hspace{46ex} Baseline~\cite{song2020score}} \hspace{-0.5ex}
    \begin{minipage}{0.47\linewidth}
        \centerline{\includegraphics[scale=0.24]{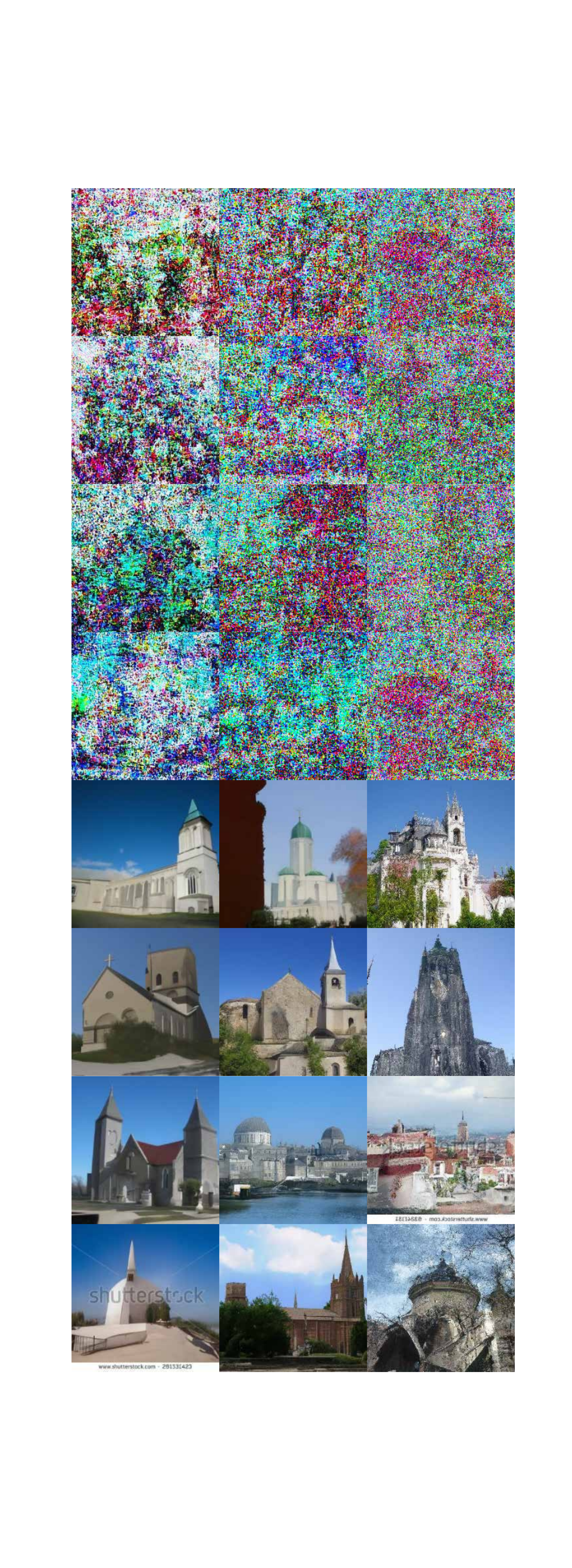}}
    \end{minipage}
    \begin{minipage}{0.47\linewidth}
            \centerline{\includegraphics[scale=0.24]{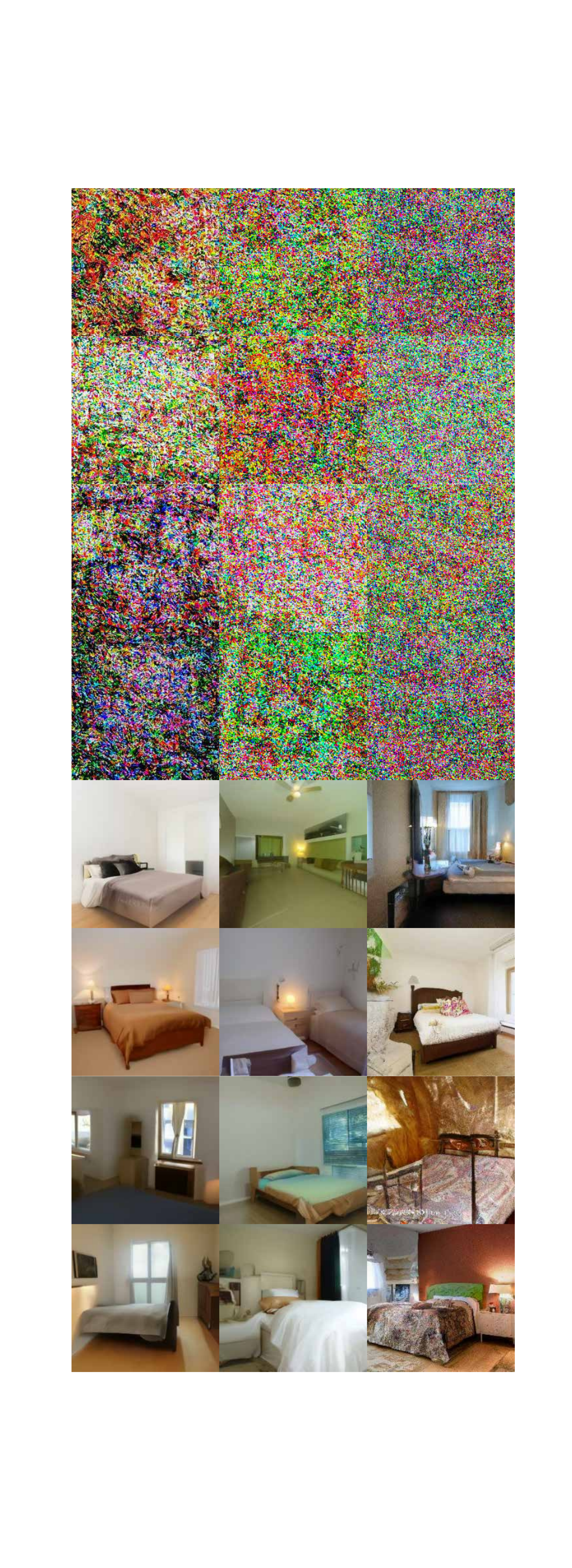}}
    \end{minipage}
    \vspace{-16ex}
           \caption{Sampling using NCSN++~\cite{song2020score} on LSUN (church and bedroom) at a resolution of $256 \times 256$ under different iteration numbers. It is observed that when the iteration number decreases, both the original method and our PDS generate samples with high-frequency noise, but the quality of the samples produced by the original method drops much more dramatically.}
        \label{fig: ncsnpp_lsun}
        \vspace{-6ex}
\end{figure}
\begin{figure}[h]
    \vspace{-1ex}
    \centering
    \begin{minipage}{1\linewidth}
        \textcolor{white}{------}\hspace{4ex}$T = 2000$\hspace{7.3ex}$T = 200$\hspace{7.3ex}$T = 133$\hspace{7.3ex}$T = 100$\hspace{7.3ex}$T = 66$
    \vspace{-7ex}
    \end{minipage}
    \begin{minipage}{1\linewidth}
        \rotatebox{90}{\textcolor{white}{------}\quad\quad\quad\quad\quad\quad Baseline~\cite{song2020score}} \hspace{-4.5ex}
        \centerline{\includegraphics[scale=0.15]{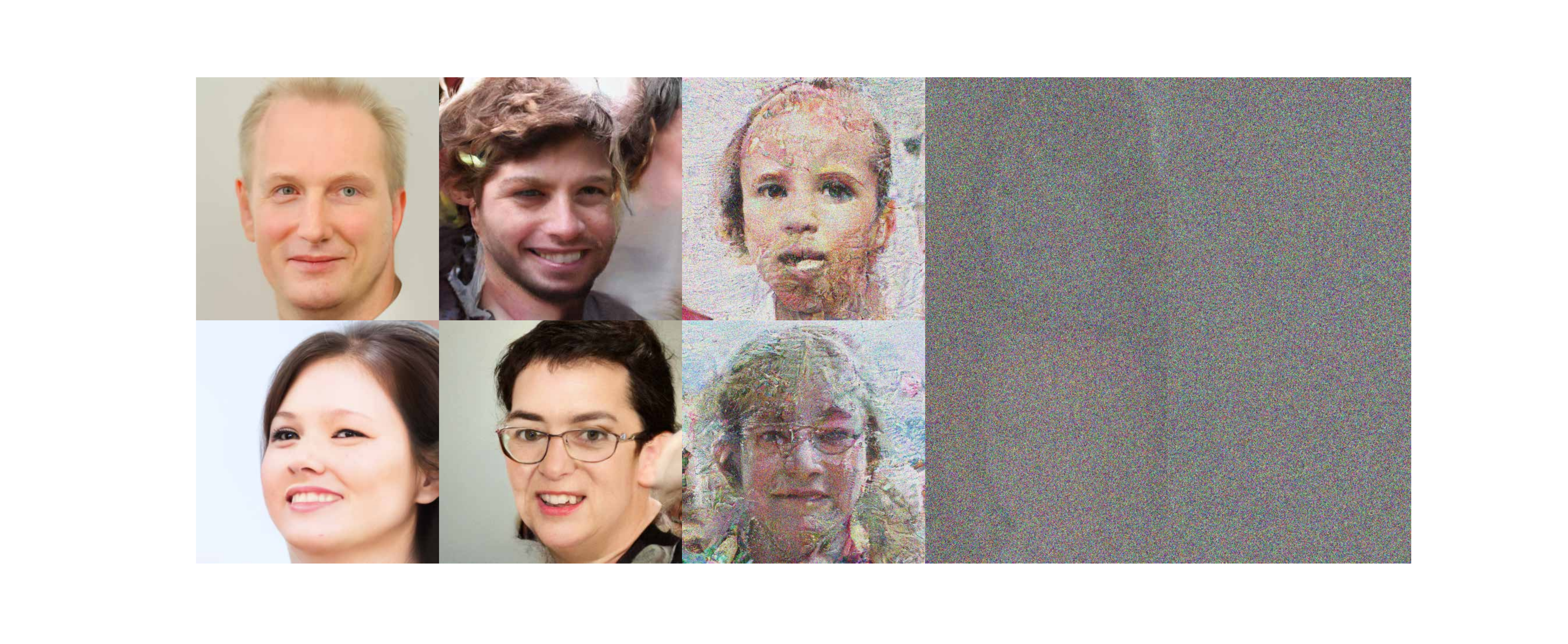}}
        \vspace{-10.75ex}
        
        \rotatebox{90}{\textcolor{white}{------} \quad\quad\quad \quad\quad\quad  Ours} \hspace{-3.69ex}
        \centerline{\includegraphics[scale=0.15]{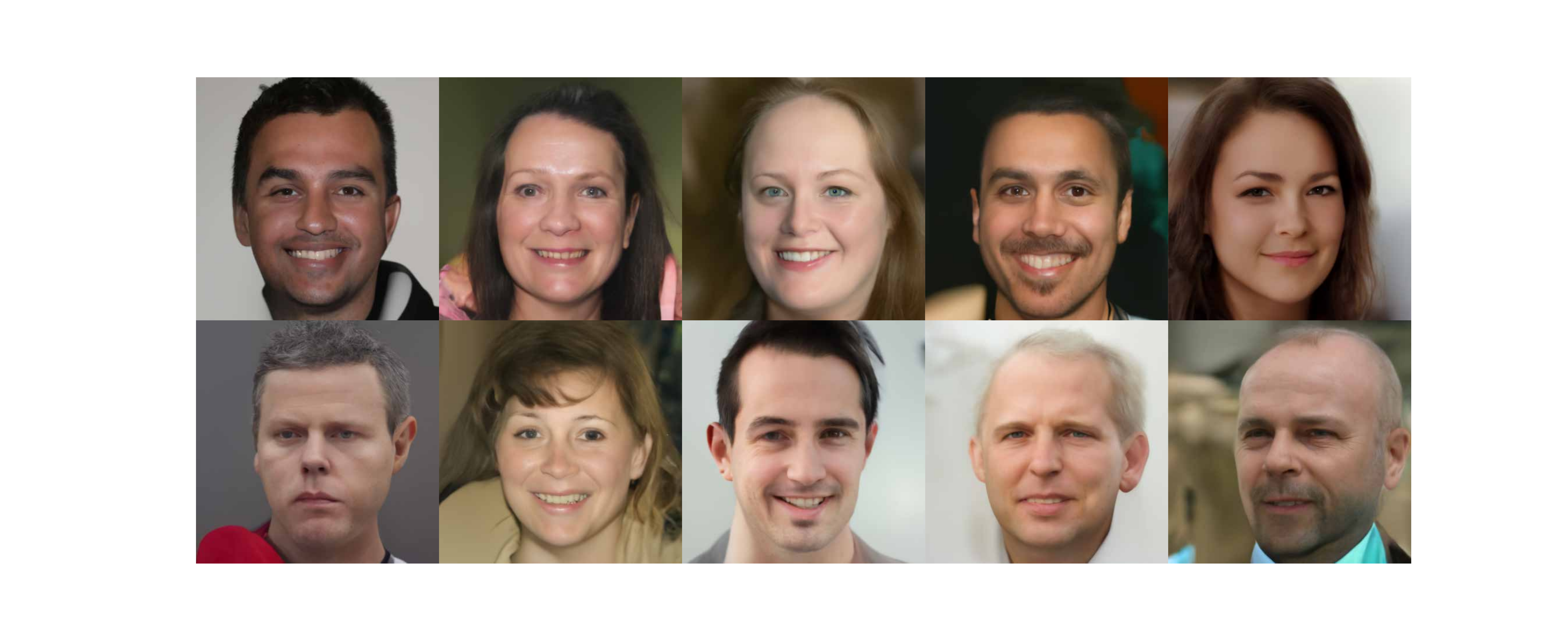}}
    \end{minipage}

    \vspace{-4ex}
        \caption{
        FFHQ~\cite{karras2019style} (facial images) at a resolution of $1024\times 1024$ generated by NCSN++~\cite{song2020score}
        under a variety of sampling iterations
        (top) without and (bottom) with our PDS. 
        It is evident that NCSN++ decades quickly with increasingly reduced sampling iterations,
        which can be well solved with PDS.
    		}
        \label{fig:ffhq_sm}
        \vspace{-4ex}
\end{figure}
\begin{figure}[h]
    \vspace{-22ex}
    \centering
    \begin{minipage}{1\linewidth}
        \centerline{\includegraphics[scale=0.26]{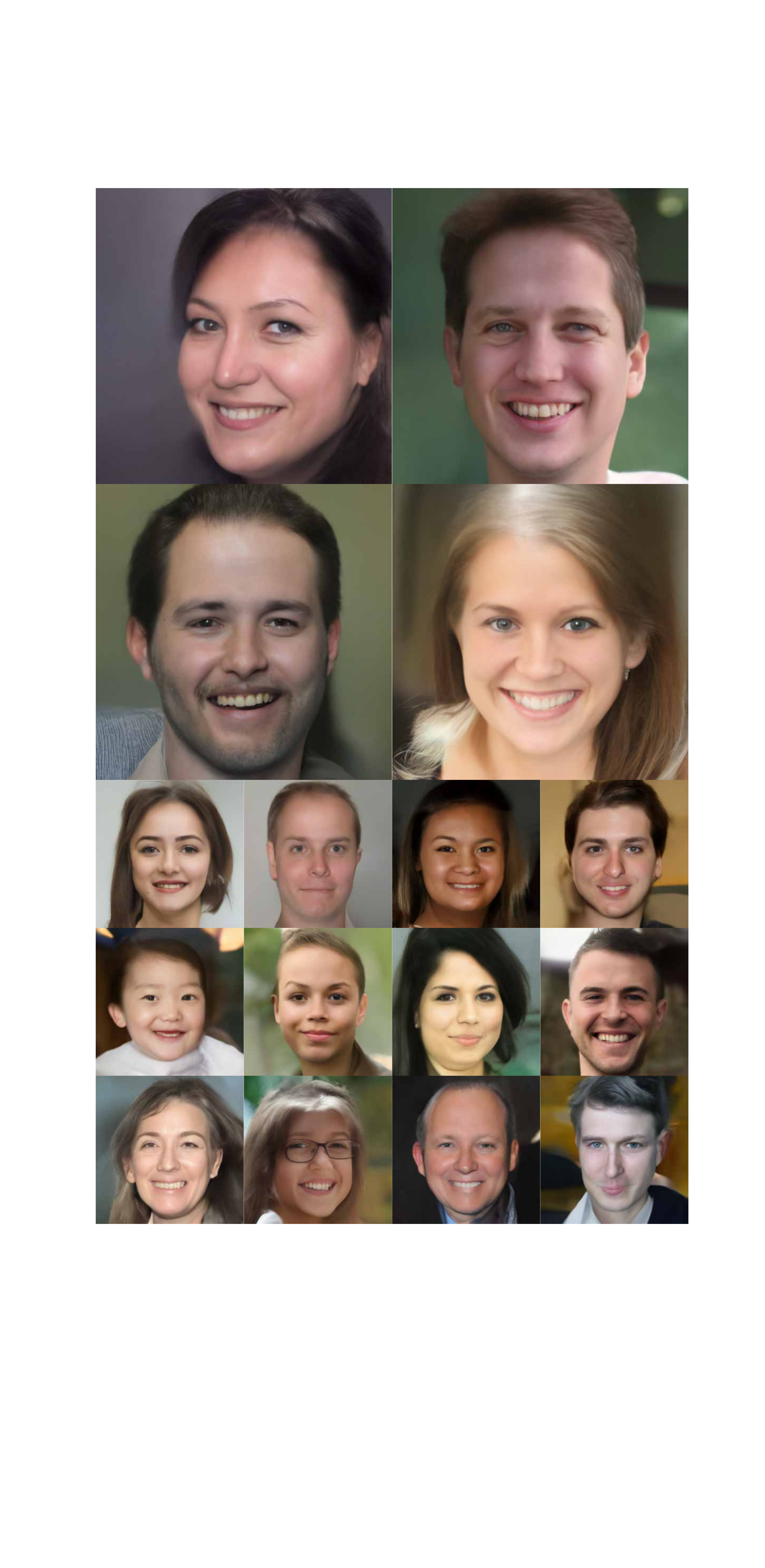}}
        \vspace{-10.75ex}
    \end{minipage}
    \vspace{-34ex}
        \caption{
        Facial images at a resolution of $1024\times 1024$ generated by NCSN++~\cite{song2020score}
        with our PDS.
        Sampling iterations: $66$.
        Dataset: FFHQ~\cite{karras2019style}.
        }
        \label{fig:ffhq_hq}
\end{figure}

\end{document}